\documentclass[conference]{IEEEtran}

\newtheorem{theorem}{Theorem}[section]

\newtheorem{assumption}[theorem]{Assumption}
\newtheorem{problem}{Problem}
\newtheorem{rem}[theorem]{Remark}

\usepackage{graphicx} % for pdf, bitmapped graphics files
\usepackage{epsfig} % for postscript graphics files
\usepackage{amsmath}
\usepackage{amssymb}
\usepackage{epstopdf}
\usepackage{cite}
\usepackage[noend,ruled,linesnumbered]{algorithm2e}
\SetKwComment{Comment}{$\triangleright$\ } {}
\usepackage{multirow}
\usepackage{rotating}
\usepackage{subfigure} 
\usepackage{color} 
\usepackage{mysymbol}
\usepackage[dvipsnames]{xcolor}
\usepackage[hyphens]{url}

\usepackage[breaklinks=true, colorlinks, bookmarks=true, citecolor=Black, urlcolor=Violet,linkcolor=Black]{hyperref}

\newcommand{\set}[1]{\left\{#1\right\}}

\usepackage{comment}
\usepackage[colorinlistoftodos,prependcaption,textwidth=1.5cm,textsize=tiny]{todonotes}
\IEEEoverridecommandlockouts                % This command is only needed if % you want to use the \thanks command

\begin{document}
\title{\LARGE \bf Technical Report: Scalable Active Information Acquisition \\for Multi-Robot Systems}

\author{Yiannis Kantaros, and George J. Pappas
\thanks{The authors are with with the GRASP Laboratory, University of Pennsylvania, Philadelphia, PA, 19104, USA. $\left\{\text{kantaros,  pappasg}\right\}$@seas.upenn.edu.  
%This material is based upon work supported by the AFRL and DARPA under Contract No. FA8750-18-C-0090.
This work was supported by the ARL grant DCIST CRA W911NF-17-2-0181.
}}
\maketitle 

\begin{abstract}
%The robots are heterogeneous in the sense that they differ in their sensing capabilities. 
%To account for sensor heterogeneity, we perform a sensor-based decomposition of the environment using weighted Voronoi diagrams.
% Then, every robot is responsible for solving a local AIA task within its assigned weighted Voronoi cell that is updated as the robot navigates the environment.
This paper proposes a novel highly scalable non-myopic planning algorithm for multi-robot Active Information Acquisition (AIA) tasks. AIA scenarios include target localization and tracking, active SLAM, surveillance, environmental monitoring and others. The objective is to compute control policies for multiple robots which minimize the accumulated uncertainty of a static hidden state over an \textit{a priori} unknown horizon. The majority of existing AIA approaches are centralized and, therefore, face scaling challenges. To mitigate this issue, we propose an online algorithm that relies on decomposing the AIA task into local tasks via a dynamic space-partitioning method. The local subtasks are formulated online and require the robots to switch between exploration and active information gathering roles depending on their functionality in the environment. The switching process is tightly integrated with optimizing information gathering giving rise to a hybrid control approach. We show that the proposed decomposition-based algorithm is probabilistically complete for homogeneous sensor teams and under linearity and Gaussian assumptions. %Extensions to account for dynamic hidden states and hidden states of unknown dimension are discussed as well. 
We provide extensive simulation results that show that the proposed algorithm can address large-scale estimation tasks that are computationally challenging to solve using existing centralized approaches.
\end{abstract}
%% In fact, this is the first time that convergence rate results are provided for sampling-based planning methods for estimation tasks. 
%%To the best of our knowledge, this is the first sampling-based planning algorithm for information gathering tasks that is highly scalable, and is also supported by formal guarantees. 
%\IEEEpeerreviewmaketitle
%%To solve the local subproblems, we employ our recently proposed sampling-based algorithm that simultaneously explores both the robot motion space and the reachable information space and can quickly compute sensor policies that achieve desired levels of uncertainty. 
   
% ------------------------------------------------------------------------------------------------------------------------------ %
\section{Introduction} \label{sec:Intro}
%env monitoring: tokekar2013tracking
%SLAM: nerurkar2014c
%source seeking: ramirez2018distributed
%sensor placement: vitus2012efficient

The Active Information Acquisition (AIA) problem has recently received considerable attention due to its wide range of applications including target tracking \cite{huang2015bank}, environmental monitoring \cite{lu2018mobile}, active simultaneous localization and mapping (SLAM) \cite{carlone2014active}, active source seeking \cite{atanasov2015distributed}, and search and rescue missions \cite{kumar2004robot}. In each of these scenarios, robots are deployed to collect information about a physical phenomenon of interest; see e.g., Figure \ref{fig:experiment}. %For example, in environmental monitoring, the goal is build spatiotemporal models of pollution propagation, and in target tracking robots are tasked with estimating and tracking the positions of targets of interest. Near-optimal methods for sensor placement and scheduling have been proposed in \cite{krause2008optimizing}. However, approaches for optimizing the trajectories of mobile sensors are either greedy and rarely provide performance guarantees, or non-myopic and optimal but computationally expensive and, therefore, not applicable to large-scale estimation tasks.  

This paper addresses the problem of designing control policies for a team of mobile homogeneous sensors which minimize the accumulated uncertainty of a static landmarks located at uncertain positions over an \textit{a priori} unknown horizon while satisfying user-specified accuracy thresholds. First, we formulate this AIA problem as a centralized stochastic optimal control problem which generates an optimal terminal horizon and a sequence of optimal control policies given measurements to be collected in the future. Under Gaussian and linearity assumptions we can convert the problem into a deterministic optimal control problem, for which optimal control policies can be designed \textit{offline}. To design  sensor policies, we propose a novel algorithm that relies on decomposing the centralized deterministic optimal control problem into local ones via dynamically tessellating the environment into Voronoi cells assuming global information about the hidden state is available. The local subproblems are formulated and solved online and require the robots to switch between exploration and information gathering roles depending on their functionality in the environment. Specifically, a robot switches to an information gathering role if landmarks are estimated to be located inside its Voronoi cell. In this case, sensor-based control actions to localize the landmarks residing within its Voronoi cell are generated by applying our recently proposed sampling-based approach that simultaneously explores both the robot motion space and the information space reachable by the sensors \cite{kantaros2019asymptotically}. On the other hand, a robot adopts an exploration role if no landmarks are estimated to be within its Voronoi region. In this case, control actions are generated using existing area coverage or exploration methods that force the robots to spread in the environment and discover landmarks \cite{Cortes_TRA04, kantaros2015distributed,leung2006active,smith2018distributed,corah2019communication,wang2020autonomous}. As the robots navigate the workspace, they update their Voronoi cells and their roles accordingly. We show that this hybrid control approach results in distributing the burden of information gathering among the robots. Also, we show that the proposed algorithm is probabilistically complete under Gaussian and linearity assumptions. We provide extensive simulation results demonstrating that our algorithm can address large scale estimation tasks that involve hundreds of robots and hidden states with hundreds of dimensions. Finally, we show that the proposed algorithm can also design sensor policies when the linearity assumptions are relaxed.%
%by introducing bias into the sampling process, the proposed algorithm can quickly design paths that achieve desired levels of uncertainty in AIA tasks that involve large teams of robots, workspaces, and dimensions of the hidden state, which is impossible using relevant methods. %\NA{We should emphasize that even though the solution is constructed under linear Gaussian assumptions, it can be used to generate control policies for nonlinear motion and observation models}

\begin{figure}[t]
  \centering
  \includegraphics[width=1\linewidth]{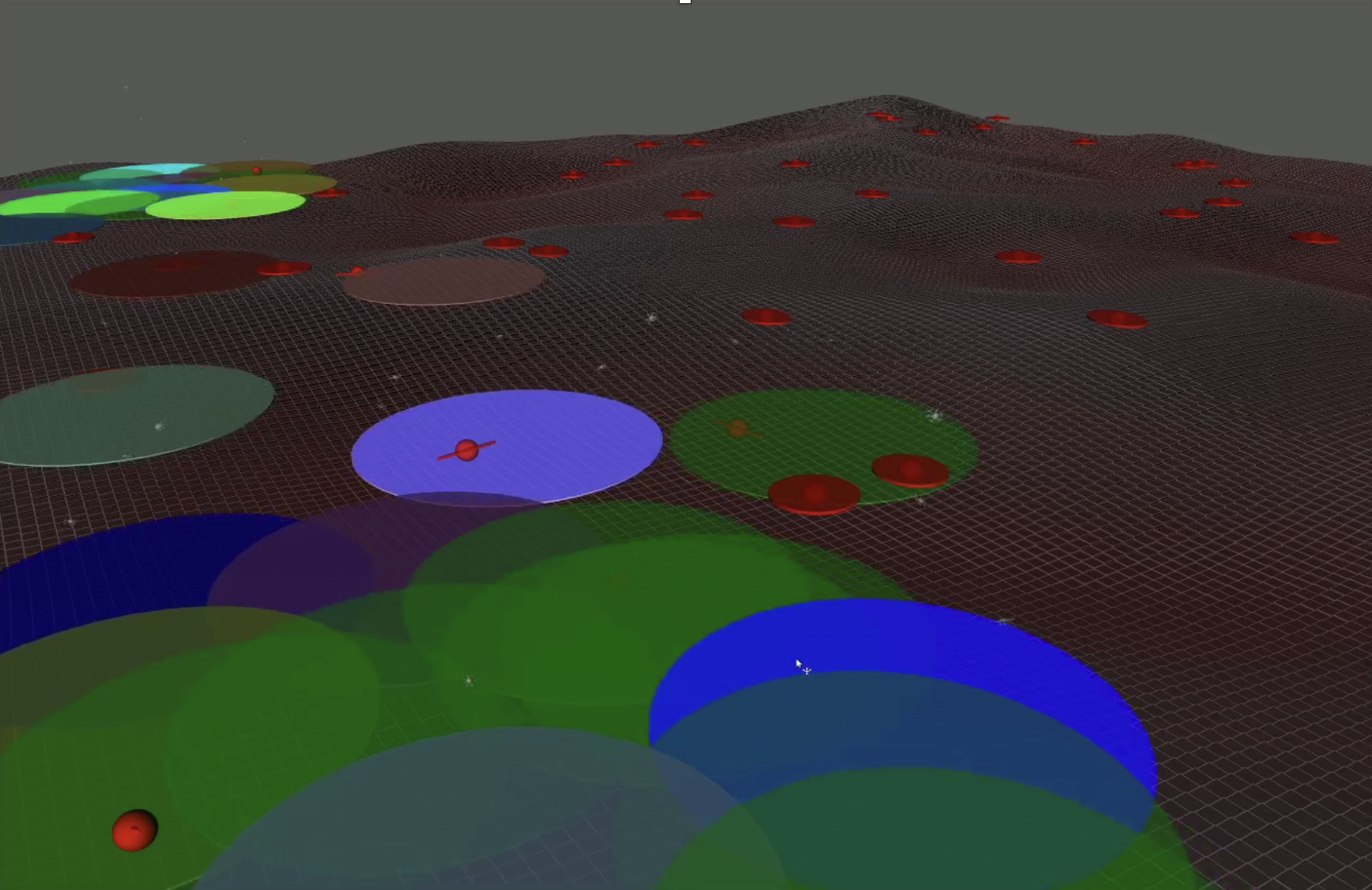}
  \caption{Landmark localization scenario: 30 UAVs with limited field-of-view (colored disks) navigating an environment to localize 50 landmarks (red spheres) of interest. Red ellipses denote the uncertainty about the landmark positions; see \cite{sim_AIA}. }
  \label{fig:experiment}
\end{figure}

%\subsection{Related works}
%Our work lies in the larger field of active state estimation. For instance, in \cite{freundlich2018distributed} a distributed optimization algorithm is proposed to control a robotic sensor network to estimate a collection of \textit{static} hidden states so that desired accuracy thresholds and confidence levels are satisfied. A feedforward value iteration algorithm is proposed in \cite{le2009trajectory} that relies on exhaustively searching both the physical workspace and the information space to design optimal open-loop control policies. This work is extended in \cite{atanasov2014information} by pruning the exploration process  while providing suboptimality guarantees. Nevertheless, these approaches are computationally intractable as the planning horizon or the number of robots increases. Decentralized AIA methods, based on coordinate descent, are proposed in \cite{singh2009efficient,atanasov2015decentralized,schlotfeldt2018anytime} with suboptimality guarantees. Nonmyopic control policies for target tracking scenarios are also presented in \cite{charrow2014approximate} that rely on approximate representations of the target's future state. Nevertheless, this approach is specifically designed for tracking a single target using small teams of mobile robots.  Finally, greedy approaches are also proposed in \cite{graham2008cooperative,dames2012decentralized, meyer2015distributed}.
{\bf{Literature Review:}} Relevant approaches to accomplish AIA tasks are typically divided into greedy and nonmyopic. Greedy approaches rely on computing controllers that incur the maximum immediate decrease of an uncertainty measure as, e.g., in \cite{martinez2006optimal,graham2008cooperative,dames2012decentralized,charrow2014approximate,meyer2015distributed}, while they are often accompanied with suboptimality guarantees due to submodular functions that quantify the informativeness of paths \cite{corah2018distributed}. Although myopic approaches are usually preferred in practice due to their computational efficiency, they often get trapped in local optima.
To mitigate the latter issue, nonmyopic \textit{search-based} approaches have been proposed that sacrifice computational efficiency in order to design optimal paths. For instance, optimal controllers can be designed by exhaustively searching the physical and the information space \cite{le2009trajectory}. More computationally efficient but suboptimal controllers have also been proposed that rely on pruning the exploration process and on addressing the information gathering problem in a \textit{decentralized} way via coordinate descent \cite{singh2009efficient,atanasov2015decentralized,schlotfeldt2018anytime}. However, these approaches become computationally intractable as the planning horizon or the number of robots increases as decisions are made locally but \textit{sequentially} across the robots. 
Nonmyopic \textit{sampling-based} approaches have also been proposed due to their ability to find feasible solutions very fast, see e.g., \cite{levine2010information,hollinger2014sampling,khodayi2019distributed,lan2016rapidly}. Common in these works is that they are \textit{centralized} and, therefore, as the number of robots or the dimensions of the hidden states increase, the state-space that needs to be explored grows exponentially and, as result, sampling-based approaches also fail to compute sensor policies because of either excessive runtime or memory requirements. More scalable but centralized sampling-based approaches that can solve target tracking problems and logic-based AIA tasks involving tens of robots and long planning horizons have also been proposed recently in \cite{kantaros2019asymptotically,kantaros2020semantic}. Scalability comparisons among search-based and sampling-based methods can be found in \cite{kantaros2019asymptotically}. In his paper, we propose a new AIA algorithm that is highly scalable due to the fact that the global AIA task is decomposed into local ones while the robots make local decisions \textit{simultaneously} as they navigate the workspace assuming all-to-all communication for information exchange purposes is available. 

The {\bf{contribution}}  of this paper can be summarized as follows. \textit{First}, we propose a new highly scalable and nonmyopic method that can quickly design control policies achieving desired levels of uncertainty in AIA tasks that involve sensor teams with hundreds of robots and hidden state with hundreds of dimensions. \textit{Second}, we show that the proposed algorithm is probabilistically complete under linearity and Gaussian assumptions. %We emphasize that this is the first AIA sampling-based algorithm supported by these theoretical guarantees. 
\textit{Third}, we provide extensive simulation results that show that the proposed method can efficiently handle large-scale estimation tasks. %which is impossible using existing centralized methods. %Moreover, our comparative simulation results revealed an important drawback of existing sampling-based approaches: exploration of large state-spaces is computationally intractable, unless the sampling operation is biased towards informative sub-spaces.

%To the best of our knowledge, this is the first highly scalable sampling-based approach for information-gathering tasks in complex environments with formal guarantees. Specifically, unlike relevant sampling-based AIA methods, the proposed algorithm is probabilistically complete, asymptotically optimal, and converges exponentially fast to the optimal solution. Moreover, we would like to highlight that this is the first time that such convergence guarantees are provided for sampling-based planning methods; see e.g., \cite{hollinger2014sampling,lan2016rapidly,levine2010information} or even sampling-based approaches for traditional point-to-point navigation \cite{karaman2011sampling}. Finally, we present extensive simulation results that show that by biasing the sampling process, the proposed framework can address large-scale estimation tasks which is not possible using available methods \red{[??]}. 

%The rest of the paper is organized as follows. In Section \ref{sec:PF}, we  formulate the multi-robot active information acquisition problem. The proposed sampling-based algorithm is presented in Section \ref{sec:samplingAlg}. Theoretical guarantees are presented in Section \ref{sec:complOpt} while simulation studies are provided in Section \ref{sec:Sim}.

% ------------------------------------------------------------------------------------------------------------------------------ %
\section{Problem Definition:\\Centralized Active Information Acquisition} \label{sec:PF}
%\input{files/PF}

% ------------------------------------------------------------------------------- %
%\subsection{Robot Motion Model}
Consider $N$ mobile robots that reside in an environment $\Omega\subset \reals^d$ with obstacles of arbitrary shape located at $O\subset\Omega$, where $d$ is the dimension of the workspace. The dynamics of the $N$ robots are described by %$\bbp_{j}(t+1)=\bbf_j(\bbp_{j}(t),\bbu_{j}(t))$, %$\bbp_{j}(t+1)=\bbf_j(\bbp_{j}(t),\bbu_{j}(t))$,
\begin{equation}\label{eq:rdynamics}
\bbp_{j}(t+1)=\bbf_j(\bbp_{j}(t),\bbu_{j}(t)),
\end{equation}
for all $j\in\ccalN:=\{1,\dots,N\}$, where $\bbp_{j}(t) \in \Omega_{\text{free}}:=\Omega \backslash O$ stands for the state (e.g., position and orientation) of robot $j$ in the obstacle-free space $\Omega_{\text{free}}$ at discrete time $t$, $\bbu_{j}(t) \in\ccalU_j$ stands for a control input in a \textit{finite} space of admissible controls $\ccalU_j$. %\NA{state that time is discrete}
%Note that the state $\bbp_{j}$ may belong to a higher dimensional space than $\Omega_{\text{free}}$. For example, this is the case if $\bbp_j(t)$ includes the position and the orientation of robot $j$. Nevertheless, throughout the rest of the paper, for simplicity of notations we denote $\bbp_{j}(t) \in \Omega_{\text{free}}$.\BS{Can we make this into a  remark or just remove it? It's extra clutter for something that is fairly straightforward.} 
%Also, note that the considered dynamics can range from simple graph-based dynamics to arbitrary nonlinear discrete dynamics\NA{This sentence could be removed}. 
Hereafter, we compactly denote the dynamics of all robots as %$\bbp(t+1)=\bbf(\bbp(t),\bbu(t)),$
\begin{equation}\label{eq:rdynamicsC}
\bbp(t+1)=\bbf(\bbp(t),\bbu(t)),
\end{equation}
where $\bbp(t)\in \Omega_{\text{free}}^N$, $\forall t\geq 0$, and $\bbu(t)\in\ccalU:=\ccalU_1\times\dots\times\ccalU_N$.

%\subsection{Estimation Task}
%The task of the robots is to collaboratively estimate a static \textit{hidden state} $\bbx$

The environment is occupied by $M>0$ \textit{static} landmarks located at uncertain positions $\bbx_i$, $i\in\ccalM:=\{1,\dots,M\}$. In particular, the robots are not aware of the true landmark positions but they have access to probability distributions over their positions. Such distributions can be user-specified or produced by Simultaneous Localization and Mapping (SLAM) methods \cite{rosinol2019kimera,bowman2017probabilistic}. Specifically, we assume that the position of landmark $i$ follows an initially known Gaussian distribution $\ccalN(\hat{\bbx}_i(0),\Sigma_i(0))$, where $\hat{\bbx}_i(0)$ and $\Sigma_i(0)$ denote the expected position and the covariance matrix at time $t=0$, respectively. Hereafter, we compactly denote the hidden state by $\bbx$, i.e., $\bbx=[\bbx_1^T,\dots,\bbx_M^T]$. Similarly, we denote by %which at any time $t\geq0$ follows a normal distribution $\ccalN(\hat{\bbx}(t),\Sigma(t))$, where 
$\hat{\bbx}(0)$ the vector that stacks the expected positions $\hat{\bbx}_i(0)$ of all landmarks and $\Sigma(0)$ the block diagonal covariance matrix where the diagonal elements are the individual matrices $\Sigma_i(0)$.
%
%governed by the following dynamics:
%\begin{equation} \label{eq:dynamics}
%\bbx(t+1) =\bbA\bbx(t) + \bbw(t),
%\end{equation}
%%\BS{As discussed with Nikolay, perhaps we should consider removing d(t) here? It just raises the question of how we know the control inputs for this system.}
%where $\bbx(t)\in\mathbb{R}^{n}$ and $\bbw(t) \in\reals^{d_w}$ denote the hidden state and the process noise at discrete time $t$, respectively. We assume that the process noise $\bbw(t)$ is normally distributed as $\bbw(t) \sim \ccalN(\bbd(t), \bbQ(t) )$, where $\bbQ(t)$ is the covariance matrix at time $t$. For instance, $\bbx(t)$ can model the position of static or mobile targets \cite{atanasov2014information}, the state of spatio-temporal fields \cite{lan2016rapidly} or gas concentration  \cite{bennetts2013towards}.

%The task of the robots is to collaboratively achieve desired levels of uncertainty about the hidden states. Assuming that the landmark positions follow Gaussian distribution, we measure the uncertainty about landmark $i$, we measure uncertainty about landmark $i$ using the determinant of its corresponding covariance matrix denoted by $\det(\Sigm)$
The robots are equipped with sensors to collect measurements associated with the hidden states as per the following stochastic linear \textit{observation model} %$\bby_j(t) = \bbM(\bbp_j(t))\bbx + \bbv(t),$ 
\begin{equation} \label{eq:measModel}
\bby_j(t) = \bbM(\bbp_j(t))\bbx + \bbv(t),
\end{equation}
where $\bby_j(t)$ is the measurement signal at discrete time $t$ taken by robot $j\in\ccalN$ associated with the landmark positions $\bbx$. Also, $\bbv(t) \sim \ccalN(\bb0, \bbR(t) )$ is sensor-state-dependent Gaussian noise with covariance $\bbR(t)$. Note that we assume that the robots are homogeneous in terms of their sensing capabilities.\footnote{This assumption is used in Section \ref{sec:complOpt} to ensure that the proposed decomposition-based approach is complete.} For simplicity of notation, hereafter, we compactly denote the observation models of all robots as  
\begin{equation}
    \bby(t)= \bar{\bbM}(\bbp(t))\bbx + \bar{\bbv}(t),~\bar{\bbv}(t) \sim \ccalN(\bb0, \bar{\bbR}(t)).
\end{equation}
%Linear observation models have been used, e.g., in \cite{bennetts2013towards} to estimate a gas concentration field.
Hereafter, we assume the robots can exchange their collected measurements via an all-to-all communication network allowing them to always maintain a common/global estimate of the hidden state. 
 %\begin{equation} \label{eq:measModelC}
%\bby(t)= \bbM(\bbp(t))\bbx + \bbv(t),~\bbv(t) \sim \ccalN(\bb0, \bbR(t) ).
%\end{equation}
%\BS{What about the linearity requirements in the measurement model? That doesn't seem to be captured now since m(x(t)) is a function of x?}
The quality of measurements taken by all robots up to a time instant $t$, collected in a vector denoted by $\bby_{0:t}$, can be evaluated using information measures, such as the mutual information between $\bby_{0:t}$ and $\bbx$, the conditional entropy of $\bbx$ given $\bby_{0:t}$, or the determinant of the a-posteriori covariance matrix, denoted by $\Sigma(t|\bby_{0:t})$. Note that the a-posteriori mean, denoted by $\hat{\bbx}(t|\bby_{0:t})$. and a-posteriori covariance matrix $\Sigma(t|\bby_{0:t})$ can be computed using probabilistic inference methods, e.g., Kalman filter. %\NA{This notation for $\Sigma(t|\bby_{0:t}))$ is not standard and should be explained. The mean $\mu$ also depends on the measurements but it is not conditioned  on them. We should also emphasize somewhere that $M$ depends on $p(t)$.} \BS{This paragraph is slightly misleading. We mention that we can evaluate uncertainty with mutual information or conditional entropy, and go on to list others, but the problem formulation is only valid for positive functions for which mutual information and entropy are not. One thing that might help is putting Remark 2.1 as a footnote with a symbol in this sentence to immediately refer the reader to this fact.}

%\subsection{Active Information Acquisition}
Given the initial robot configuration $\bbp(0)$ and the Gaussian distributions $\ccalN(\bbx_i(0),\Sigma_i(0))$ for the hidden states, our goal is to select a \textit{finite} horizon $F\geq 0$ and compute control inputs $\bbu(t)$, for all time instants $t\in\{0,\dots,F\}$, that solve the following stochastic optimal control problem: %\NA{It is not clear how $\Sigma(t|\bby_{0:t})$ is obtained below and why the target motion model does not appear as a constraint}
%
%$$ \abs{\bbx_t - \bar{\bbx}_{ij,t}} \to 0 \text{ as } t \to \infty , \ \forall \bbq \in \Omega, \forall r_{ij}. $$
%
\begin{subequations}
\label{eq:Prob}
\begin{align}
& \min_{\substack{ F, \bbu_{0:F}}} \left[J(F,\bbu_{0:F},\bby_{0:F}) = \sum_{t=0}^{F}   \det\Sigma(t|\bby_{0:t}) \right]  \label{obj1}\\
%&\ \ \st \ \ \ \ \ \ \ \ \ \ \ \ \bbp(F) \in \ccalP_{\text{goal}}, \label{constr1}\\
&\ \ \ \ \ \ \ \ \ \ \ \  \ \ \ \ \ \ \ \det\Sigma(F|\bby_{0:F})\leq \delta, \label{constr2} \\
%&\ \ \ \ \ \ \ \ \ \ \ \  \ \ \ \ \ \ \  F\leq \bar{F},\nonumber\\
& \ \ \ \ \ \ \ \ \ \ \ \  \ \ \ \ \ \ \ \bbp(t) \in \Omega_{\text{free}}^N,  \label{constr3} \\
&\ \ \ \ \ \ \ \ \ \ \ \  \ \ \ \ \ \ \ \bbp(t+1) = \bbf(\bbp(t),\bbu(t)), \label{constr4} \\
%&\ \ \ \ \ \ \ \ \ \ \ \  \ \ \ \ \ \ \ \bbx(t+1) =\bbA\bbx(t) + \bbw(t), \label{constrTrgt}\\%\bbx(t+1) =g(\bbx(t), \bbw(t)), \label{constrTrgt}\\
&\ \ \ \ \ \ \ \ \ \ \ \  \ \ \ \ \ \ \ \bby(t) =\bar{\bbM}(\bbp(t))\bbx + \bar{\bbv}(t), \label{constr5}
\end{align}
\end{subequations}
where the constraints hold for all time instants $t\in\{0,\dots,F\}$.
In \eqref{obj1}, $\bbu_{0:F}$ stands for the sequence of control inputs applied from $t=0$ until $t=F$.
%\begin{subequations}
%\label{eq:Prob}
%\begin{align}
%& \min_{\substack{ F, \bbu_{0:F}}} \left[J(F,\bbu_{0:F},\bby_{0:F}) = \sum_{t=0}^{F}   %\det\Sigma(t|\bby_{0:t}) \right]  \label{obj1}\\
%%&\ \ \st \ \ \ \ \ \ \ \ \ \ \ \ \bbp(F) \in \ccalP_{\text{goal}}, \label{constr1}\\
%&\ \ \ \ \ \ \ \ \ \det\Sigma(F|\bby_{0:F})\leq \delta, \label{constr2} \\
%%&\ \ \ \ \ \ \ \ \ \ \ \  \ \ \ \ \ \ \  F\leq \bar{F},\nonumber\\
%& \ \ \ \ \ \ \ \ \ \bbp(t) \in \Omega_{\text{free}}^N,  \label{constr3} \\
%&\ \ \ \ \ \ \ \ \ \ \bbp(t+1) = \bbf(\bbp(t),\bbu(t)), \label{constr4} \\
%&\ \ \ \ \ \ \ \ \ \  \bbx(t+1) =g(\bbx(t), \bbw(t)), \label{constrTrgt}\\
%&\ \ \ \ \ \ \ \ \ \ \bby(t) =\bbM(\bbp(t))\bbx(t) + \bbv(t), \label{constr5}
%\end{align}
%\end{subequations}
%
% \{\{\bbu_j(t)\}_{j=1}^{N}
%%\Sigma_i(t)
In words, the objective function \eqref{obj1} captures the cumulative uncertainty in the estimation of $\bbx$ after fusing information collected by all robots from $t=0$ up to time $F$. 
%Moreover, the first constraint \eqref{constr1} requires that at time $F$ the configuration of the robots is within $\ccalP_{\text{goal}}\subseteq \Omega_{\text{free}}^N$; see also Remark \ref{rem1}. 
The first constraint \eqref{constr2} requires the terminal uncertainty of $\bbx$ to be below a user-specified threshold $\delta$. %\NA{The entropy $H(\bbx(F)|\bby_{0:F})$ is proportional to $\log\det\Sigma(F|\bby_{0:F})$. Why not make this explicit and simplfiy (4b) by just having a constraint on $\det\Sigma(F|\bby_{0:F})$?}.
The second constraint \eqref{constr3} requires that the robots should never collide with obstacles. The last two constraints capture the robot dynamics and the sensor model.
The Active Information Acquisition (AIA) problem in \eqref{eq:Prob} is a stochastic optimal control problem for which, in general, closed-loop control policies are optimal. Nevertheless, given  the linear observation models, we can apply the separation principle presented in \cite{atanasov2014information} to convert 
\eqref{eq:Prob} to the following deterministic optimal control problem. %\NA{Why is the conditioning on $y_{0:t}$ removed below?}
%Deterministic optimal control problem:
\begin{subequations}
\label{eq:Prob2}
\begin{align}
& \min_{\substack{ F, \bbu_{0:F}}} \left[J(F,\bbu_{0:F}) = \sum_{t=0}^{F}  \det\Sigma(t) \right] \label{obj2}\\
%&\ \ \st \ \ \ \ \ \ \ \ \ \ \ \ \bbp(F) \in \ccalP_{\text{goal}}, \label{constr21}\\
& \ \ \ \ \ \ \ \ \ \ \ \ \ \ \ \ \ \ \det\Sigma(F)\leq \delta,  \label{constr22} \\
&\ \ \ \ \ \ \ \ \ \ \ \ \ \ \ \ \ \ \  \bbp(t) \in \Omega_{\text{free}}^N, \label{constr23} \\
& \ \ \ \ \ \ \ \ \ \ \ \ \ \ \ \ \ \ \  \bbp(t+1) = \bbf(\bbp(t),\bbu(t)), \label{constr24}\\
& \ \ \ \ \ \ \ \ \ \ \ \ \ \ \ \ \ \ \ \Sigma(t+1) =\rho(\bbp(t+1),\Sigma(t)), \label{constr25}
\end{align}
\end{subequations}%\BS{Can we use the Problem script here? And not put equation numbers on each line? Rather the whole thing should read Problem 2:}
where $\rho(\cdot)$ stands for the Kalman Filter Ricatti map. Note that open loop (offline) policies are optimal solutions to \eqref{eq:Prob2}. The problem addressed in this paper can be summarized as follows.%\NA{Since entropy is proportional to $\log\det\Sigma(F|\bby_{0:F})$ and since the cost function is positive-additive, the optimal planning horizon is the first hitting time of the condition $H(\bbx(F))\leq \delta$, which as mentioned earlier can just be written in terms of $\det\Sigma(F|\bby_{0:F}) \leq \delta'$. This is a planning problem with goal specified by $\det\Sigma(F|\bby_{0:F}) \leq \delta'$ and stage cost given by $\det\Sigma(t)$. I would relate the proposed biased sampling procedure to a heuristic function for this problem. For example, can you relate the "minimize the distance to the target" intuition to a lower bound on the objective $\sum_{t=0}^{F}  \det\Sigma(t)$?} %; \red{see also Figure ??}.
\begin{problem} \label{prob}(Active Information Acquisition)
Given an initial robot configuration $\bbp(0)$ and a Gaussian prior distribution $\ccalN(\hat{\bbx}(0),\Sigma(0))$ for the hidden state $\bbx$, select a horizon $F$ and compute control inputs $\bbu(t)$ for all time instants $t\in\{0,\dots,F\}$ as per \eqref{eq:Prob2}.
\end{problem}

Throughout the paper we make the following assumption allows for application of a Kalman filter; see \cite{atanasov2014information}.% and the underlying communication network between the robots.
\begin{assumption} \label{as:assumption}%[State Model]
The measurement noise covariance matrices $\bbR(t)$ are known for all time $t\geq 0$. %Furthermore, the state dynamics \eqref{eq:dynamics} and sensing model \eqref{eq:measModel} are differentiable functions. 
%Finally, we assume that all robots can communicate with each other at any time $t$. 
\end{assumption}

%\begin{rem}[Nonlinear Models]
%Note that we assume that the dynamics \eqref{eq:dynamics} and the sensing model \eqref{eq:measModel} are modeled as linear systems. Also, in \eqref{eq:rdynamics} we assume noiseless robot dynamics. As discussed before, 
%Note that Assumption \ref{as:assumption} allows for offline computation of optimal policies, since the solution to \eqref{eq:Prob2} does not depend on the robot measurements. %In Section \ref{sec:Sim}, we present numerical experiments where this assumption is relaxed.
%\end{rem}
\begin{rem}[Optimal Control Problem \eqref{eq:Prob}]\label{rem1}
In \eqref{eq:Prob}, any other optimality metric, not necessarily information-based, can be used in place of \eqref{obj1} as long as it is always positive. If non-positive metrics are selected, e.g., the entropy of $\bbx$, then \eqref{eq:Prob} is not well-defined, since the  optimal terminal horizon $F$ is infinite. On the other hand, in the first constraint \eqref{constr2}, any uncertainty measure can be used without any restrictions, e.g., scalar functions of the covariance matrix, or mutual information. Moreover, note that without the terminal constraint \eqref{constr2}, the optimal solution of \eqref{eq:Prob} is  all robots to stay put, i.e., $F=0$. 
%Additional terminal constraints can be added to \eqref{eq:Prob}, such as, $\bbp(F)\in\ccalP_{\text{goal}}\subseteq\Omega_{\text{free}}^N$, to model joint task planning and estimation scenarios, where, e.g., the robots should eventually visit a base station to upload an estimate of the hidden state with user-specified accuracy determined by $\delta$.
\end{rem}

% ------------------------------------------------------------------------------------------------------------------------------ %
\section{Distributed Planning for\\ Active Information Acquisition} \label{sec:samplingAlg}

A centralized and offline solution to \eqref{eq:Prob2} is proposed in \cite{kantaros2019asymptotically} that is shown to be optimal under linearity and Gaussian assumptions; nevertheless, a centralized solution can incur a high computational cost as it requires exploring a high dimensional joint space composed of the space of multi-robot states and covariance matrices. As a result, its computational cost increases as the number of robots and landmarks increases. To mitigate these challenges, we propose a distributed and online - but sub-optimal - approach that allows the robots to make local decisions, assuming all-to-all communication, to solve \eqref{eq:Prob2} as they navigate the environment. 

%In Section \ref{sec:decomp}, we show how \eqref{eq:Prob2} can be decomposed into local problems while a computationally efficient method to solve them is presented in Section \ref{sec:sol}

\subsection{Decomposition of AIA Task \& Robot Roles}\label{sec:decomp}

The proposed algorithm is summarized in Algorithm \ref{alg:AIA}. The key idea relies on decomposing \eqref{eq:Prob2} into local optimal control problems, that are formulated and solved online to generate individual robot actions. Decomposition of \eqref{eq:Prob2} into local problems is attained by assigning landmarks to each robot online via a \textit{dynamic} space-partitioning method. 

Specifically, at any time $t\geq0$, the environment is decomposed into Voronoi cells generated by the robot positions $\bbp_j(t)$. The Voronoi cell assigned to robot $j$ at time $t$, denoted by $V_j(t)$, is defined as $ V_j(t)=\set{\bbq\in\Omega~|~\lVert\bbp_j(t)-\bbq\rVert\leq \lVert\bbp_e(t)-\bbq\rVert,~\forall e\neq j}$,
%\begin{equation}
%    V_j(t)=\set{\bbq\in\Omega~|~\lVert\bbp_j(t)-\bbq\rVert\leq \lVert\bbp_e(t)-\bbq\rVert,~\forall e\neq j},
%\end{equation}
i.e., $V_j(t)$ collects all locations $\bbq$ in the workspace that are closer to robot $j$ than to any other robot $e\neq j$ at time $t\geq 0$. Observe that it holds that $\bigcup_{j\in\ccalN}V_j(t)=\Omega$. %Also, note that the each robot can construct its individual Voronoi cell in a distributed fashion by communicating only with its Delaunay neighbors  \cite{Cortes_SIAMJCO05}; see Remark \ref{rem:com}. 

Given $V_j(t)$, robot $j$ is responsible for \textit{actively decreasing the uncertainty of the landmarks that are expected/estimated to lie within its Voronoi cell}. Formally, the landmarks that robot $j$ is responsible for at time $t$ are collected in the following set:
\begin{equation}\label{eq:Aj}
\ccalA_j(t)=\{i\in\ccalM ~|~(\hat{\bbx}_i(t)\in\ccalV_j(t))\wedge (\det\Sigma_j(t)>\delta) \}.
\end{equation}
In words, $\ccalA_j(t)$ collects the landmarks that have not been localized with accuracy determined by $\delta$ and are estimated to be closer to robot $j$ than to any other robot. By definition of the set of $\ccalA_j(t)$ and the fact that $\bigcup_{j\in\ccalN}V_j(t)=\Omega$, we have that every landmark is assigned to exactly one robot and that there are no unassigned landmarks, i.e., $\bigcup_{j\in\ccalN}\ccalA_j(t)=\ccalM$. Notice that the set $\ccalA_j(t)$ may be empty in case the estimated positions of all landmarks are outside the $j$-th Voronoi cell. Depending on the emptiness of the sets $\ccalA_j(t)$, the robots switch between \textit{exploration} and \textit{AIA} roles in the workspace, as follows. 

\subsubsection{Local AIA Strategy} If at time $t=\bar{t}$ there are landmarks assigned to robot $j$, i.e., $\ccalA_j(\bar{t})\neq\emptyset$, then robot $j$ solves the following local AIA problem:
\begin{subequations}
\label{eq:LocalProbb2}
\begin{align}
& \min_{\substack{ F_j, \bbu_{j,\bar{t}:F_j}}} \left[J(F_j,\bbu_{\bar{t}:F_j}) = \sum_{t=\bar{t}}^{F_j}  \sum_{i\in\ccalA_j(\bar{t})}\det\Sigma_i(t) \right] \label{obj3}\\
%&\ \ \st \ \ \ \ \ \ \ \ \ \ \ \ \bbp(F) \in \ccalP_{\text{goal}}, \label{constr21}\\
& \ \ \ \ \ \ \ \ \ \ \ \ \ \ \ \ \ \ \det\Sigma_i(F_j)\leq \delta, \forall i\in\ccalA_j(\bar{t}) \label{constr32} \\
&\ \ \ \ \ \ \ \ \ \ \ \ \ \ \ \ \ \ \  \bbp_j(t) \in \Omega_{\text{free}}, \label{constr33} \\ % V_j(\bar{t})
& \ \ \ \ \ \ \ \ \ \ \ \ \ \ \ \ \ \ \  \bbp_j(t+1) = \bbf_j(\bbp_j(t),\bbu_j(t)), \label{constr34}\\
& \ \ \ \ \ \ \ \ \ \ \ \ \ \ \ \ \ \ \ \Sigma(t+1) =\rho(\bbp_j(t+1),\Sigma(t)), \label{constr35}
\end{align}
\end{subequations}
where the constraints \eqref{constr33}-\eqref{constr35} hold for all time instants $t\in[\bar{t},F_j]$. Note that \eqref{eq:LocalProbb2} can viewed as local version of \eqref{eq:Prob2}. In words, \eqref{eq:LocalProbb2} requires robot $j$ to compute a terminal horizon $F_j$ and a sequence of control inputs over this horizon, denoted by  $\bbu_{j,\bar{t}:F_j}$, so that (i) the accumulated uncertainty over the assigned targets is minimized (see \eqref{obj3}), (ii) the terminal uncertainty of the assigned targets is below a user-specified threshold (see \eqref{constr32}) while  (iii) avoiding obstacles and respecting the robot dynamics and the Kalman filter Ricatti equation (see \eqref{constr33}-\eqref{constr35}). Note that in \eqref{constr35}, the  Kalman filter Ricatti equation is applied using only the position of robot $j$ and not the multi-robot state as opposed to \eqref{constr25}. In other words, in \eqref{constr35}, the covariance matrix $\Sigma(t)$ is computed using only local information. For simplicity, we do not add any dependence to robot $j$ to $\Sigma(t)$.
Equivalence between the centralized AIA problem \eqref{eq:Prob2} and the local AIA problems in \eqref{eq:LocalProbb2} is discussed in Section \ref{sec:complOpt}. A computationally efficient method to quickly solve the nonlinear optimal control problem \eqref{eq:LocalProbb2} is discussed in Section \ref{sec:aia}.

%[EXPLAIN here that you further decompose it into sequential problems across the assigned targets or in the simulations]

\subsubsection{Exploration Strategy} If at time $t=\bar{t}$, there are no landmarks assigned to robot $j$, i.e., $\ccalA_j(\bar{t})=\emptyset$, then robot $j$ is responsible for \textit{exploring} the environment so that eventually it holds that $\ccalA_j(t)\neq\emptyset$. %, i.e., robot $j$ becomes responsible for localizing a landmark.
In this way, the burden of localizing the landmarks is shared among the robots. The latter can be accomplished if these robots adopt an exploration/mapping or an area coverage strategy. For instance, in \cite{yamauchi1997frontier,leung2006active} a frontiers-based exploration strategy is proposed. Specifically, dummy
“exploration” landmarks with locations at the current map frontiers are considered with a known Gaussian prior on their locations. This fake uncertainty in the exploration-landmark locations promises information gain to the robots. In \cite{kantaros2015distributed,Cortes_ESAIMCOCV05}, area coverage methods are proposed assuming robots with range-limited sensors. Specifically, the robots follow a gradient-based policy to maximize an area coverage objective that captures the part of the workspace that lies within the sensing range of all robots. Alternatively, robots with $\ccalA_j(\bar{t})=\emptyset$ can be recruited from robots $e$ with $\ccalA_e(\bar{t})\neq\emptyset$ to help them localize the corresponding landmarks.

%its Voronoi cell to check if there are any landmarks inside it. By definition of $\ccalA_j(t)$, it holds that a landmark may be inside $\ccalV_j(t)$ even though $\ccalA_j(t)=\emptyset$. As an exploration strategy, we employ a frontiers-based approach [???]; however, any exploration algorithm can be employed [???]. Specifically, for all robots $j$, with $\ccalV_j(t)\neq\emptyset$, we introduce dummy “exploration” landmarks with locations $[\bbq_1,\dots,\bbq_{N_d}]$ at the current map frontiers that lie within $\ccalV_j(t)$ and specify a Gaussian prior on their locations with mean $\hat{\bbq}=\bbq$ and block diagonal covariance $\Sigma$ with $N_q$ blocks. This fake uncertainty in the exploration-landmark locations promises information gain to the robots. Specifically, navigating the environment to decrease the uncertainty of the frontier-landmarks, by solving the corresponding local optimal control problem \eqref{eq:LocalProbb2}, allows the robots to explore their Voronoi cells and discover landmarks with inaccurate position estimates.

\subsection{Local Planning for Non-myopic AIA}\label{sec:aia}
In this section, we present a computationally efficient method to solve the local optimal control problem \eqref{eq:LocalProbb2}. %Note that \eqref{eq:LocalProbb2} is a nonlinear optimal control problem that requires to exhaustively search over the joint space of the robot states and covariance matrices, a task that is particularly challenging to accomplish online. To mitigate this, 
Specifically, to solve \eqref{eq:LocalProbb2}, we employ our recently proposed sampling-based algorithm for AIA tasks \cite{kantaros2019asymptotically}. In particular, the employed algorithm relies on incrementally constructing a directed tree that explores both the information space and the physical space. 
%\subsection{Decomposition of \eqref{eq:LocalProbb2} across the landmarks}

%\subsection{Sampling-based Planning for AIA}
In what follows, we denote by $\mathcal{G}_j=\{\mathcal{V}_j,\mathcal{E}_j,J_{\ccalG,j}\}$ the tree constructed by robot $j$ to solve \eqref{eq:LocalProbb2}, where $\ccalV_j$ is the set of nodes and $\ccalE_j\subseteq \ccalV_j\times\ccalV_j$ denotes the set of edges. The set of nodes $\mathcal{V}_j$ contains states of the form $\bbq_j(t)=[\bbp_j(t), \Sigma(t)]$. %\footnote{Throughout the paper, when it is clear from the context, we drop the dependence of $\bbq(t)$ on $t$.} 
The function $J_{\ccalG}:\ccalV_j\rightarrow\mathbb{R}_{+}$ assigns the cost of reaching node $\bbq_j\in\mathcal{V}_j$ from the root of the tree. The root of the tree, denoted by $\bbq_j(0)$, is constructed so that it matches the initial state  $\bbp_j(0)$ of robot $j$ and the prior covariance $\Sigma(0)$, i.e., $\bbq_j(0)=[\bbp_j(0), \Sigma(0)]$. For simplicity of notation, hereafter we drop the dependence of the tree on robot $j$. 
The cost of the root $\bbq(0)$ is $J_{\ccalG}(\bbq(0)) = \det\Sigma(0)$,
while the cost of a node $\bbq(t+1)=[\bbp_j(t+1), \Sigma(t+1)]\in\ccalV$, given its parent node $\bbq(t)=[\bbp_j(t), \Sigma(t)]\in\ccalV$, is computed as $J_{\ccalG}(\bbq(t+1))= J_{\ccalG}(\bbq_j(t)) +  \det\Sigma(t+1).$
%\begin{equation}\label{eq:costUpd}
%J_{\ccalG}(\bbq(t+1))= J_{\ccalG}(\bbq_j(t)) +  \det\Sigma(t+1).
%\end{equation}
Observe that by applying this cost function recursively, we get that $J_{\ccalG}(\bbq(t+1)) = J(t,\bbu_{0:t+1})$ which is the objective function in \eqref{eq:LocalProbb2}.

\begin{algorithm}[t]
\footnotesize
\caption{Local Planning for Global AIA Task}
\LinesNumbered
\label{alg:AIA}
\KwIn{ (i) robot dynamics; observation model; (iii) prior Gaussian $\ccalN(\hat{\bbx}(0),\Sigma(0))$; (iv) initial robot configuration $\bbp_j(0)$}
\KwOut{Terminal horizon $F$, and control inputs $\bbu_{0:F}$};\label{rrt:init}
$t=0$\label{alg:init1}\;
Compute Voronoi cell $V_j(0)$ and set of assigned targets $\ccalA_j(0)$\;\label{alg:init1}
\While{$\exists j~\text{s.t.}~\det\Sigma_j(t)>\delta$}{
\If{$\ccalA_j(t)\neq\emptyset$}{
\If{$(t=0)\vee(\ccalA_j(t)\neq \ccalA_j(t-1))$}{
Compute sequence of controllers $\bbu_{j,t:F_j}$ for AIA \cite{kantaros2019asymptotically}\;\label{aia:cntrl1} Apply control input $\bbu_j(t)$\;}
\Else{
Apply the next control input from the most recently computed sequence of control inputs\;} 
}
\Else{
Compute next control input $\bbu_j(t)$ to explore $V_j(t)$\;\label{aia:cntrl2}
}
Collect sensor measurements\;
Update $\hat{\bbx}(t+1)$ and $\Sigma(t+1)$\; 
Update Voronoi cell $V_j(t+1)$ and set $\ccalA_j(t+1)$\;
$t=t+1$\;\label{aia:vor}
}
\end{algorithm}

\normalsize
The tree $\ccalG$ is initialized so that $\ccalV=\{\bbq(0)\}$, $\ccalE=\emptyset$, and $J_{\ccalG}(\bbq(0)) = \det\Sigma(0)$. Also, the tree is built incrementally by adding new states $\bbq_\text{new}$ to $\ccalV$ and corresponding edges to $\ccalE$, at every iteration $n$ of the sampling-based algorithm, based on a \textit{sampling} and \textit{extending-the-tree} operation. 
After taking $n_{\text{max}}\geq 0$ samples, where $n_{\text{max}}$ is user-specified, the sampling-based algorithm terminates and returns a solution to Problem \ref{prob}, i.e., a terminal horizon $F$ and a sequence of control inputs $\bbu_{0:F}$. 
To extract such a solution, we need first to define the set $\ccalX_g\subseteq\ccalV$ that collects all states $\bbq(t)=[\bbp_j(t), \Sigma(t)]\in\ccalV$ of the tree that satisfy %$\bbp(t)\in\ccalP_{\text{goal}}$ and %%\eqref{constr21}-
$\det\Sigma(F_j)\leq \delta$, for all landmarks that belong to $\ccalA_j(t)$, which is the constraint \eqref{constr32}. Then, among all nodes $\ccalX_g$, we select the node $\bbq(t)\in\ccalX_g$, with the smallest cost $J_{\ccalG}(\bbq(t))$, denoted by $\bbq(t_{\text{end}})$. %\NA{Note that if you relate the biasing to consistent heuristic the search can stop as soon as you find the first terminal state and you will not need to compare all such states. A consistent heuristic allows biasing the search even more by inflating the heurstic (multiplying by a constant $\epsilon$) and still guaranteeing $\epsilon$-suboptimality}. 
Then, the terminal horizon is $F=t_{\text{end}}$, and the control inputs $\bbu_{0:F}$ are recovered by computing the path $\bbq_{0:t_{\text{end}}}$ in $\ccalG$ that connects $\bbq(t_{\text{end}})$ to the root $\bbq(0)$, i.e., $\bbq_{0:t_{\text{end}}}= \bbq(0), \dots, \bbq(t_{\text{end}})$. 
Note that satisfaction of the constraints \eqref{constr33}-\eqref{constr35} is guaranteed by construction of $\ccalG$; see \cite{kantaros2019asymptotically}. The core operations of this algorithm, `\textit{sample}' and `\textit{extend}' that are used to incrementally construct the tree $\ccalG$ can be found in \cite{kantaros2019asymptotically}. %are described in Appendix \ref{appB}.

\subsection{Overview of Distributed Hybrid AIA Algorithm}\label{sec:overview}
At any time $t\geq0$ every robot $j$ adopts a role in the environment based on the set $\ccalA_j(t)$. %, i.e., the set of landmarks that are assigned to them. 
If $\ccalA_j(t)=\emptyset$, then robot $j$ adopts an exploration mode aiming to eventually become responsible for localizing landmarks. If $\ccalA_j(t)\neq\emptyset$, then robot $j$ solves \eqref{eq:LocalProbb2} to actively decrease the uncertainty of the assigned landmarks. In both cases, robot $j$ generates a sequence of controls denoted by $\bbu_{j,t:F_j}$. %As the robots apply their control inputs 
Once the robots apply their respective control input $\bbu_j(t)$, they communicate via an all-to-all connected network to share their measurements and positions so that they can (i) maintain a global estimate of the hidden state and (ii) update their Voronoi cells $V_j(t+1)$. Based on the new Voronoi cells, the robots locally compute the new sets of assigned targets $\ccalA_j(t+1)$. The robots for which it holds that $\ccalA_j(t+1)=\ccalA_j(t)$ do not update their corresponding sequence of control actions, i.e., they apply the control input $\bbu_j(t+1)$ computed at time $t$. On the other hand, the robots that satisfy $\ccalA_j(t+1)\neq\ccalA_j(t)$ recompute their control actions to accomplish either their assigned AIA or exploration tasks. %, depending on the emptiness of $\ccalA_j(t+1)$.

%\begin{rem}[Communication Requirements]\label{rem:com}
%Note that in general it holds that $\bbx_j(t)\neq \hat{\bbx}_j(t)$ which means that a robot may be responsible for a target that does not lie within its Voronoi cell. Moreover, depending on the multi-robot configuration it is possible that a robot may not be responsible for any targets, i.e., $\ccalA_j(t)=\emptyset$...
%To implement Algorithm \ref{alg:AIA}, all-to-all communication is required to (i) share the collected measurements so that all robots have access to the same estimate for the hidden states; (ii) compute the Voronoi cells.
%\textcolor{blue}{Algorithm \ref{alg:AIA} is distributed in the sense each robot computes locally and online its own control actions by solving the corresponding local problems \eqref{eq:LocalProbb2}. Also, to implement Algorithm \ref{alg:AIA}, all-to-all communication is required as discussed in Section \ref{sec:overview}.}
%Also, to implement Algorithm \ref{alg:AIA}, all-to-all communication is required so that the robots can (i) share their collected measurements to the have the same global estimate for the hidden states and (ii) exchange their positions to compute their respective Voronoi cells.}
%\end{rem}

\subsection{Extensions}\label{sec:extensions}

%\subsubsection{Unknown Number of Landmarks}
%In Section \ref{sec:PF}, we assume that the number of landmarks is known. This  can be relaxed by enforcing all robots with $\ccalA_j(t)=\emptyset$ to adopt an exploration/mapping strategy \cite{leung2006active,smith2018distributed,corah2019communication,wang2020autonomous}. In this case, Algorithm \ref{alg:AIA} terminates when the whole environment has been explored and all landmarks that have been localized with user-specified accuracy $\delta$.

\subsubsection{Mobile Landmarks}
Mobile landmarks can also be considered that are governed by the following known but noisy dynamics: $\bbx_i(t+1) =\bbg_i(\bbx_i(t),\bba_i(t), \bbw_i(t))$,
%\begin{equation} \label{eq:dynamics}
%\bbx_i(t+1) =\bbg_i(\bbx_i(t),\bba_i(t), \bbw_i(t)), 
%\end{equation}
where $\bba_i(t)$ and $\bbw_i(t) \in\reals^{d_w}$ are the control input and the process noise for landmark $i$ at discrete time $t$. We assume that the process noise $\bbw_i(t)$ is uncertain and follows a known normal distribution, i.e., $\bbw(t) \sim \ccalN(0, \bbQ(t))$, where $\bbQ(t)$ is the covariance matrix at time $t$. Assuming that the control inputs $\bba_i(t)$ and the covariance matrices $\bbQ(t)$ are known to the robots for all $t\geq 0$, \cite{kantaros2019asymptotically} can be used to design AIA paths.\footnote{Note that in this case, Algorithm \ref{alg:AIA} is not complete; see also the proof of Theorem \ref{thm:probCompl2}.} %The only difference is that \eqref{eq:predict} will be used to construct the landmark states $\hat{\bbx}_{\text{new}}(t)$, instead of \eqref{eq:updX}, which are then used to define the map $\ccalM_{\text{new}}(t)$.

\subsubsection{Nonlinear Observation Models}
In Section \ref{sec:PF}, we assumed a sensor model that is linear with respect to the landmark positions allowing us to compute offline the a-posteriori covariance matrix without the need of measurements. This assumption can be relaxed by computing the a-posteriori covariance matrices using the linearized observation model about the estimated landmark positions; see Section \ref{sec:Sim}.

%Nevertheless, this assumption may not hold in practice. In this case, during the execution of the sampling-based AIA algorithm \cite{kantaros2019asymptotically},  we compute the a-posteriori covariance matrices using the linearized observation model about the estimated landmark positions; see Section \ref{sec:Sim}.

%\subsection{Open Loop Control is Optimal}

\section{Completeness, Optimality \& Convergence}\label{sec:complOpt}
%In this section, we examine the completeness of Algorithm \ref{alg:AIA}. To show this, we first need to state the following result borrowed from \cite{kantaros2019asymptotically}. %Specifically, in Theorem \ref{thm:probCompl} and \ref{thm:asOpt}, we show that Algorithm \ref{alg:RRT} is probabilistically complete and asymptotically optimal. Also, in Theorem \ref{thm:conv}, we show that Algorithm \ref{alg:RRT} converges exponentially fast to the optimal solution of Problem \ref{prob}. 
In this section, we show that Algorithm \ref{alg:AIA} is complete  for homogeneous sensor networks and static hidden states, i.e., if there there exists a solution to \eqref{eq:Prob2}, then Algorithm \ref{alg:AIA} will find it. This result is formally stated in Theorem \ref{thm:probCompl2}; to show this, we need first to state the following theorem.

\begin{theorem}[Completeness of Sampling-based AIA \cite{kantaros2019asymptotically}]\label{thm:probCompl}
The sampling-based AIA algorithm described in Section \ref{sec:aia} is probabilistically complete, i.e., if there exists a solution to \eqref{eq:LocalProbb2}, then it will find with probability $1$ a path $\bbq_{j,0:F_j}$, defined as a sequence of states in $\ccalV_j$, i.e., $\bbq_{j,0:F_j} = \bbq_j(0), \bbq_j(1), \bbq_j(2), \dots, \bbq_j(F_j)$, 
%\begin{equation}\label{eq:path}
%\bbq_{0:F} = \bbq(0), \bbq(1), \bbq(2), \dots, \bbq(F),
%\end{equation}
that solves \eqref{eq:LocalProbb2}. %, where $\bbq_j(t)\in\ccalV_j$, for all $t\in\{0,\dots,F_j\}$.
\end{theorem}

\begin{theorem}[Completeness of Distributed AIA Algorithm]\label{thm:probCompl2}
%If the exploration strategy forces robots to explore the whole workspace (\textcolor{red}{make this more formal})
%Consider any robot state $\bbp_j(t)$ and corresponding Voronoi cell $V_j(t)$ at any time instant $t\geq0$. 
Assume that (i) given the initial multi-robot state $\bbp(0)$ and the prior distributions $\ccalN(\hat{\bbx}_i(0),\Sigma_i(0))$, there exists a solution to the centralized optimal control problem \eqref{eq:Prob2}; (ii) the robot dynamics allow all robots to reach any obstacle-free location in $\Omega$. %; (iii) $\bbx_i\in V_j(t)$ for $j\in\ccalA_j(t)$, for all robots $i\in\ccalN$. 
Under assumptions (i)-(ii), Algorithm \ref{alg:AIA} is probabilistically complete, i.e., at any time instant $t$ it will generate a terminal horizon $F_j$ and a sequence of control actions $\bbu_{j,t:F_j}$ for all robots with $\ccalA_j(t)\neq\emptyset$ that solve the respective local problems \eqref{eq:LocalProbb2}. 
\end{theorem}

\begin{proof}
By assumption (i) we have that there exists a solution to the centralized optimal control problem \eqref{eq:Prob2}, i.e., there exists a finite sequence of multi-robot states defined as $\bbp_{0:F}=\bbp(0),\dots,\bbp(F)$ that solves \eqref{eq:Prob2}, where  $\bbp(t)=[\bbp_1(t),\bbp_2(t),\dots,\bbp_N(t)]\in\mathbb{R}^{N\times d}$. First, we show that if there exists such a feasible solution to \eqref{eq:Prob2}, then there exists a feasible sequence that solves \eqref{eq:LocalProbb2} constructed at $t=0$ for any robot $j$. Next, we generalize this result for any time instant $t>0$. Then, we conclude that Algorithm \ref{alg:AIA} is probabilistically complete due to Theorem \ref{thm:probCompl}.

In particular, first, using $\bbp_{0:F}$, we construct a sequence of waypoints, denoted by $\bbp_{0:H}^j$, for any given robot $j$ that solves the corresponding  sub-problem \eqref{eq:LocalProbb2} defined at time $t=0$. This sequence is constructed so that robot $j$ visits all single-robot waypoints $\bbp_r(t)$ that appear in $\bbp_{0:F}$ including those waypoints that refer to robots $r\neq j$, for all $t\in[0,\dots,H]$.
%Formally, this sequence, denoted by $\bbp_{0:H}^j$, is defined as  
%\begin{align*}\label{eq:local}
%\bbp_{0:H}^j=& \underbrace{\bbp_{1,0}(0),\bbp_{2,0}(1),\dots, \bbp_{N,0}(N-1)}_{=\bbp(0)}, \underbrace{\bbp_{1,1}(N),\bbp_{2,1}(N+1),\dots, \bbp_{2N,1}(2N-1)}_{=\bbp(1)},\dots,\\
%&\underbrace{\bbp_{1,F}(FN),\bbp_{2,F}(FN+1),\dots, \bbp_{N,F}((F+1)N-1)}_{=\bbp(F)}
%\end{align*}
%where (i) $\bbp_{r,t'}(t)\in\Omega_{\text{free}}$ refers to the waypoint that robot $j$ should visit at time $t$ which coincide with the position that robot $r$ had at time instant $t'=???$ in the sequence $\bbp_{0:F}$ and (ii) the terminal horizon $H$ is defined as $H=(F+1)N-1$.  
In other words, $\bbp_{0:H}^j$ is constructed by unfolding $\bbp_{0:F}$ and writing it as a sequence of single-robot positions. For instance, the first $N$ waypoints in $\bbp_{0:H}^j$ constitute the multi-robot state $\bbp(0)$ in $\bbp_{0:F}$. Similarly, the second group of $N$ waypoints in $\bbp_{0:H}^j$ constitutes the multi-robot state $\bbp(1)$ in $\bbp_{0:F}$. The same logic applies to all subsequent groups of $N$ waypoints. Thus, the terminal horizon $H$ is $H=(F+1)N-1$. Next we claim that if a robot $j$ follows the path $\bbp_{0:H}^j$, then at time $H$ the terminal uncertainty of \textit{all} landmarks $i$ will satisfy $\det \Sigma_i(H)\leq \delta$ (regardless of the motion of other robots). To show this, first note that during the execution of the multi-robot sequence $\bbp_{0:F}$, at time $t$ the Kalman filter Riccati map is applied sequentially $N$ times for each observation model/robot. On the other hand, during the execution of the single-robot sequence $\bbp_{0:H}^j$, at time $t$ the Riccati map is executed only once as per the observation model of robot $j$. Since (a) all robots have the same observation model, and (b) the hidden state (i.e., the positions of the landmarks) is static, we conclude, by definition of the Kalman filter Riccati map, that the terminal uncertainty of the hidden state when $\bbp_{0:H}^j$ is executed will be the same as the one when $\bbp_{0:F}$ is executed. Since $\bbp_{0:F}$ is a feasible solution to \eqref{eq:Prob2}, then after robot $j$ following $\bbp_{0:H}^j$, we have that $\det \Sigma_i(H)\leq \delta$, for all landmarks $i$, even for those that are not in $\ccalA_j(0)$. Also, by assumption (ii), the path $\bbp_{0:H}^j$ respects the robot dynamics. As a result, $\bbp_{0:H}^j$ is a feasible solution to \eqref{eq:LocalProbb2} constructed at $t=0$.  

Finally, we inductively show that if there exists a solution to  \eqref{eq:LocalProbb2} constructed at $t=0$, there exists a solution to all local problems constructed at future time instants $t>0$. To show this, let $t_1>0$ be the first time instant (after $t=0$) when a new  sub-problem in the form of \eqref{eq:LocalProbb2} is defined. Observe that if at time $t=0$, robot $j$ decides to follow a feasible path $\bbp_{0:H}^j$, not necessarily the path constructed before, then at any time $t>0$, continuing following the corresponding sub-path $\bbp_{t:H}^j$, trivially results in achieving the desired terminal uncertainty of \textit{all} landmarks - including those that belong $\ccalA_j(t_1)$ - by feasibility of $\bbp_{0:H}^j$. 
Thus, a feasible solution to \eqref{eq:LocalProbb2} formulated at $t_1$ is the sub-path $\bbp_{t_1:H}^j$ (i.e., a sub-sequence of $\bbp_{0:H}^j$).\footnote{Note that the Voronoi cells are used only for landmark assignment purposes while robot mobility is not restricted within them. Thus, updating the Voronoi cells does not affect feasibility of the paths $\bbp_{t_n:H}^j$.} Following the same logic inductively for all future time instants $t_n$ when \eqref{eq:LocalProbb2} is reformulated, we conclude that there always exists a solution to the all local sub-problems completing the proof.
%
%%%%%%%%%%%%%%%+++++++++++++++
%By applying this logic to all landmarks, we conclude that at any time $t$, a feasible solution to all sub-problems \eqref{eq:LocalProbb2} exists and will be found due to Theorem \ref{thm:probCompl} completing the proof.
%Also, as discussed in \ref{sec:aia}, since $\bigcup_{j\in \ccalN}V_j(t)=\Omega$, we have that every landmark is assigned to a unique robot. Thus, for every landmark $i$ 
\end{proof}

\begin{rem}[Online vs Offline Planning]\label{rem:onlineRe}
According to Theorem \ref{thm:probCompl2}, Algorithm \ref{alg:AIA} can find at time $t=0$ feasible (offline) paths that solve the centralized problem \eqref{eq:LocalProbb2}. As it will be shown in Section \ref{sec:Sim}, the benefit of dynamically assigning landmarks to robots and accordingly re-planning online is that the burden of localizing all $M$ landmarks is shared among all robots resulting in shorter terminal horizons.
\end{rem}

\begin{rem}[Implementation]\label{rem:implem}
A more computationally efficient approach to apply \cite{kantaros2019asymptotically} is to define the goal region $\ccalX_g$ so that it requires only one landmark - instead of all - in $\ccalA_j(t)$ to be localized with accuracy $\delta$. Such an approach may not yield the optimal solution to \eqref{eq:LocalProbb2} due to its myopic nature but it does not sacrifice completeness. This can be shown by following the same logic as in the proof of Theorem \ref{thm:probCompl2}.
%Specifically, let $i$ be the landmark that robot $j$ aims to localize after applying a sequence of control actions denoted by $\bbu_{j,t:\bar{t}}^{\ell_i}$ generated by \cite{kantaros2019asymptotically}. Also, assume that given $V_j(t)$ there exists a solution to \eqref{eq:LocalProbb2}, i.e., there exist measurement locations that if robot $j$ visits then desired levels of uncertainty will be achieved for all landmarks in $\ccalA_j(t)$. All these locations can still be visited assuming simple robot dynamics (e.g., $\bbp_j(t+1) =\bbp_j(t) +\bbu_j(t)$) after robot $j$ applies the control actions in $\bbu_{j,t:\bar{t}}^{\ell_i}$. In fact, a path towards them will be found due to the completeness nature of the AIA algorithm in \cite{kantaros2019asymptotically}. Recall that the time instants when measurement-locations are visited are irrelevant since the landmarks are assumed to be static. 
%a sequence of control inputs $\bbu_{j,t:F_j}$ that solves
\end{rem}

\begin{rem}[Optimality]
Note that to solve the local problem \eqref{eq:LocalProbb2}, information that may be collected by the rest of the robot team is neglected; see also Section \ref{sec:overview}. As a result, the synthesized paths may not constitute an optimal solution to \eqref{eq:Prob2}. This local design of paths sacrifices optimality but it allows for addressing large-scale estimation tasks. % that are computationally challenging to solve using existing centralized methods due to the large-state space that needs to be  explored.}
%However, this information will be taken into account during the execution of the locally synthesized paths since then the robots share their measurements and collectively update their beliefs about the hidden state. 
%Note that Voronoi diagrams are not equitable and, therefore, a landmark may be sensed by more than one robot. O
\end{rem}

%\section{Multi-Robot Multi-Target Tracking}\label{sec:targetTr}
%\input{files/targetTracking}

\section{Numerical Experiments} \label{sec:Sim}
\begin{table*}[]
\caption{Scalability Analysis: The runtimes $T_{\text{plan}}$, $T_{\text{vor}}$, and $T_{\text{total}}$ are in secs, secs, and mins, respectively. }
\label{tab:scale}
\centering
\begin{tabular}{|l|l|l|l|l|l|l|l|l|l|l|l|}
\hline
                                                            & \begin{tabular}[c]{@{}l@{}}N=10\\ M=10\end{tabular} & \begin{tabular}[c]{@{}l@{}}N=20\\ M=10\end{tabular} & \begin{tabular}[c]{@{}l@{}}N=50\\ M=10\end{tabular} & \begin{tabular}[c]{@{}l@{}}N=100\\ M=10\end{tabular} & \begin{tabular}[c]{@{}l@{}}N=10\\ M=100\end{tabular} & \begin{tabular}[c]{@{}l@{}}N=20\\ M=100\end{tabular} & \begin{tabular}[c]{@{}l@{}}N=50\\ M=100\end{tabular} & \begin{tabular}[c]{@{}l@{}}N=100\\ M=100\end{tabular} & \begin{tabular}[c]{@{}l@{}}N=300\\ M=100\end{tabular} & \begin{tabular}[c]{@{}l@{}}N=100\\ M=200\end{tabular} & \begin{tabular}[c]{@{}l@{}}N=300\\ M=200\end{tabular} \\ \hline
\begin{tabular}[c]{@{}l@{}}$T_{\text{plan}}$/$T_{\text{vor}}$\end{tabular} & 0.11/0.007                                          & 0.12/0.002                                          & 0.13/0.02                                         & 0.12/0.05                                           & 0.18/0.009                                           & 0.19/0.014                                            & 0.17/0.03                                            & 0.20/0.05                                             & 0.26/0.36                                             & 0.31/0.06                                             & 0.36/0.19                                             \\ \hline
$T_{\text{total}}$                                            & 0.7                                            & 1.7                                             & 5.1                                              & 9.2                                           & 2.9                                         & 3.5                                       & 7.8                                            & 17.8                                           &  78.1                                        & 30.6                                              & 98.4                                           \\ \hline
$F$                                             & 38$\pm$5                                            &34$\pm$2                                            & 32$\pm$2                                               & 30$\pm$4                                        & 82$\pm$5                                          & 70$\pm$4                                             & 63$\pm$2                                              & 60$\pm$2                                              & 58$\pm$3                                            & 68$\pm$3                                               & 58$\pm$7                                           \\ \hline
\end{tabular}
\end{table*}

In this section, we present numerical experiments that illustrate the performance of Algorithm \ref{alg:AIA} and show that it can solve large-scale estimation tasks that are computationally challenging to solve using existing centralized methods; see e.g., \cite{kantaros2019asymptotically}. 
Specifically, first, we examine the scalability of Algorithm \ref{alg:AIA} for various numbers of robots and landmarks. 
We also illustrate the benefit of online re-planning
%show that dynamically updating the targets assigned to the robots and accordingly re-planning results in significant performance gains in terms of the terminal horizon (i.e., the time required to solve the corresponding AIA task) 
compared to offline planning (see Remark \ref{rem:onlineRe}) and investigate the effect of communication on the terminal horizon. Numerical experiments related to mobile landmarks are also provided.
 % such as greedy, (nonmyopic) optimal, or other sampling-based approaches; see Section \ref{sec:comp}. %Finally, we examine the optimality of Algorithm \ref{alg:RRT}; see Section \ref{sec:opt}.
The sampling-based AIA algorithm \cite{kantaros2019asymptotically} is implemented as discussed in Remark \ref{rem:implem} using the biased density functions $f_{\ccalV}$ and $f_{\ccalU}$ designed in \cite{kantaros2019asymptotically} with parameters $p_{\ccalV}=p_{\ccalU}=0.9$. All case studies have been implemented using MATLAB 2016b on a computer with Intel Core i7 3.1GHz and 16Gb RAM. Simulation videos can be found in \cite{sim_AIA}.

%\subsection{Robot Dynamics \& Sensors}\label{sec:rmodel}
%Throughout this section, we consider robots that reside in a non-convex environment $\Omega\subset\mathbb{R}^2$ with obstacles; see e.g., Figure \ref{fig:traj15s}. Also, 

{\bf{Robot Dynamics \& Sensors:}} Throughout this section, we consider robots with differential drive dynamics, where $\bbp_j(t)$ captures both the position and the orientation of the robots. Specifically, the available motion primitives are $u\in\{0,0.1\}\text{m/s}$ and $\omega\in\set{0,5, 10, \dots, 350, 355}\text{deg/s}$.
%
%The linear dynamics are defined as  
%\begin{equation}\label{eq:linRbt}
%\bbp_j(t+1)=\bbp_j(t)+\bbu_j(t),
%\end{equation}
%where $\bbp_j(t)\in\mathbb{R}^2$ denotes the position of robot $j$, and 
%$\left\lVert\bbu_j(t) \right\rVert\in\{0, 0.2\}$, for all robots $j$.
%The nonlinear dynamics are defined as 
%\footnotesize{
%\begin{equation}\label{eq:nonlinRbt}
%\begin{bmatrix}
%   p_j^1(t+1) \\
%  p_j^2(t+1)\\
%  \theta_j(t+1)
%  \end{bmatrix}=  
%  \begin{bmatrix}
%   p_j^1(t) \\
%  p_j^2(t)\\
%  \theta_j(t)
%  \end{bmatrix}+ 
%   
%   \begin{cases}
%        \begin{bmatrix}
%   \tau u \cos(\theta_j(t)+\tau\omega/2) \\
%   \tau u \sin(\theta_j(t)+\tau\omega/2)\\
%  \tau  \omega
%  \end{bmatrix}, \text{if} ~\tau\omega<0.001\\
%  \\
% \begin{bmatrix}
%   \frac{u}{\omega} (\sin(\theta_j(t)+\tau\omega)-\sin(\theta_j(t))) \\
%   \frac{u}{\omega} (\cos(\theta_j(t)) - \cos(\theta_j(t)+\tau\omega)) \\
%  \tau \omega
%  \end{bmatrix}, \text{else},
%    \end{cases}      
%\end{equation}}
%\normalsize
%where $u\in\{0,0.2\}\text{m/s}$ and $\omega\in\set{0,\pm \pi/4,\pm \pi/2, \pm \pi/1.33, \pm \pi}$. Note in \eqref{eq:nonlinRbt}, that the state $\bbp_j(t)$ includes both the position $[p_j^1(t),p_j^2(t)]^T$ and the orientation  $\theta_j(t)$ of the robots, i.e., $\bbp_j(t)=[p_j^1(t),p_j^2(t),\theta_j(t)]^T$.
%
Moreover, we assume that the robots are equipped with omnidirectional, range-only, sensors with limited range of $0.4$m while they reside in $10\times10$ workspace. Every robot $j$ can generate measurements associated with landmark $i$ as per the model $y_{j,i} = \ell_{j,i}(t) + v(t) \ ~\mbox{if}~ \ (\ell_{j,i}(t) \leq 2)$,
%\begin{equation}\label{eq:meassim}
%y_{j,i} = \ell_{j,i}(t) + v(t) \ ~\mbox{if}~ \ (\ell_{j,i}(t) \leq 2) \wedge (i\in\texttt{FOV}_j).
%\end{equation}
%
where $\ell_{j,i}(t)$ is the distance between landmark $i$ and robot $j$, and $v(t) \sim \ccalN(0,\sigma^2(\bbp_j(t),\bbx_i(t)))$ is the measurement noise. Also, we model the measurement noise so that $\sigma$ increases linearly with $\ell_{j,i}(t)$, with slope $0.25$, as long as $\ell_{j,i}(t)\leq 2$; if $\ell_{j,i}(t)> 2$, then $\sigma$ is infinite. Observe that this observation model is nonlinear and, therefore, the separation principle, discussed in Section \ref{sec:PF}, does not hold. Thus, we execute the sampling-based algorithm to solve the local sub-problems \eqref{eq:LocalProbb2} using the linearized observation model about the estimated landmark positions. Also, the robots with no assigned landmarks, i.e., the robots with $\ccalA_j(t)=\emptyset$, navigate the environment so that they maximize an area coverage metric that captures the  region that lies within the sensing range of all robots \cite{Cortes_ESAIMCOCV05}. The latter forces these robots to spread in the environment so that landmarks eventually reside within their region of responsibility distributing the burden of active information gathering across the robots. %; see also Section \ref{sec:vor}.

\begin{figure}[t]
  \centering
%\subfigure[$N/M=1/5$, Cost = $34.47$, $F=374$,  Runtime = 23.74secs]{
%\label{fig:traj15f}
%\includegraphics[width=0.45\linewidth]{figures/traj_1vs5_fast.pdf}}
%
\subfigure[$t=0$]{
    \label{fig:t0}
  \includegraphics[width=0.45\linewidth]{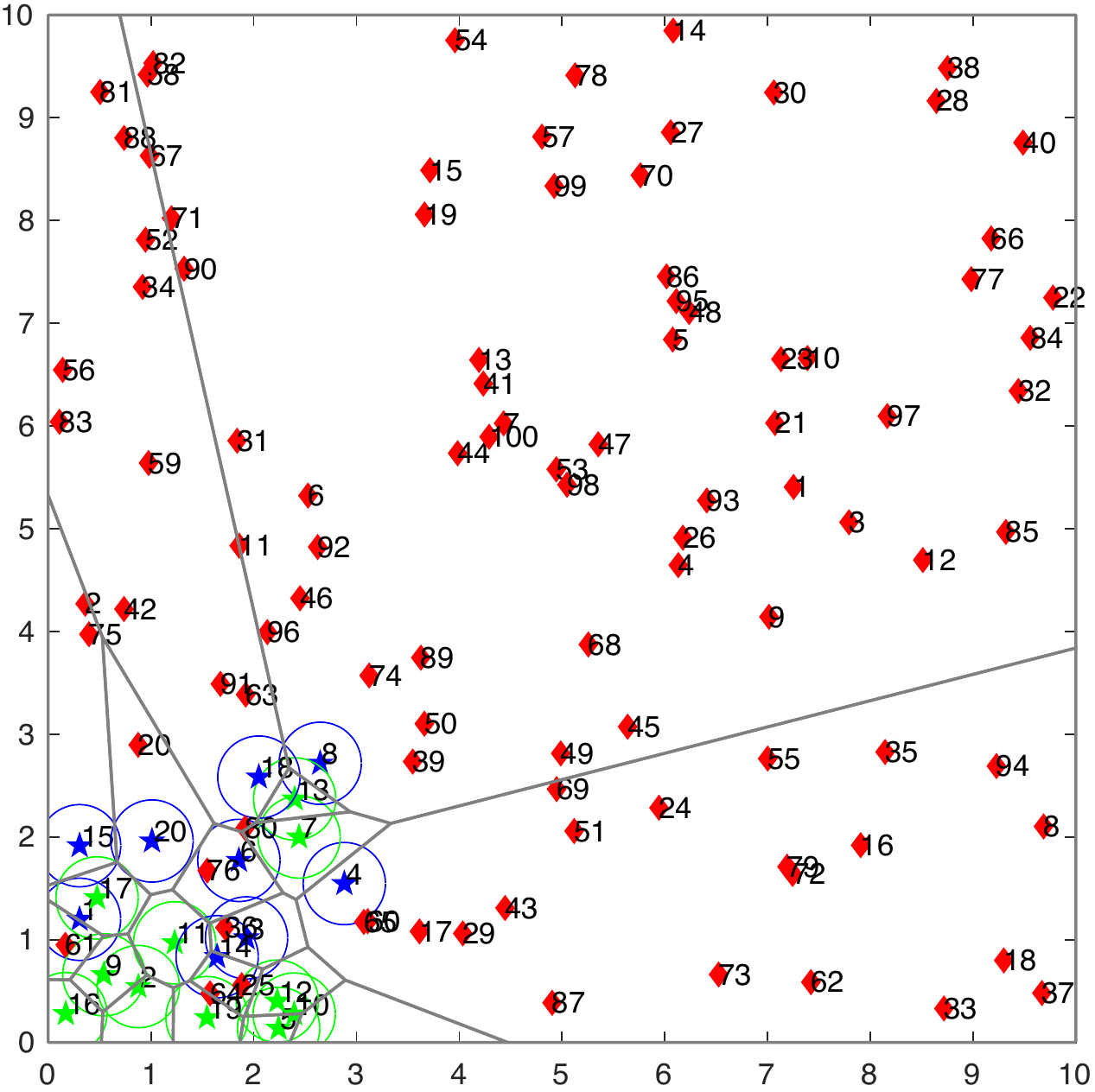}}
 %
% \subfigure[$N/M=1/5$, Cost = $75.03$, $F=374$,  Runtime = 4.19 mins]{
%    \label{fig:biasedGreedy15}
%  \includegraphics[width=0.45\linewidth]{figures/biasedGreedy.pdf}}
   \subfigure[$t=30$]{
    \label{fig:t30}
  \includegraphics[width=0.45\linewidth]{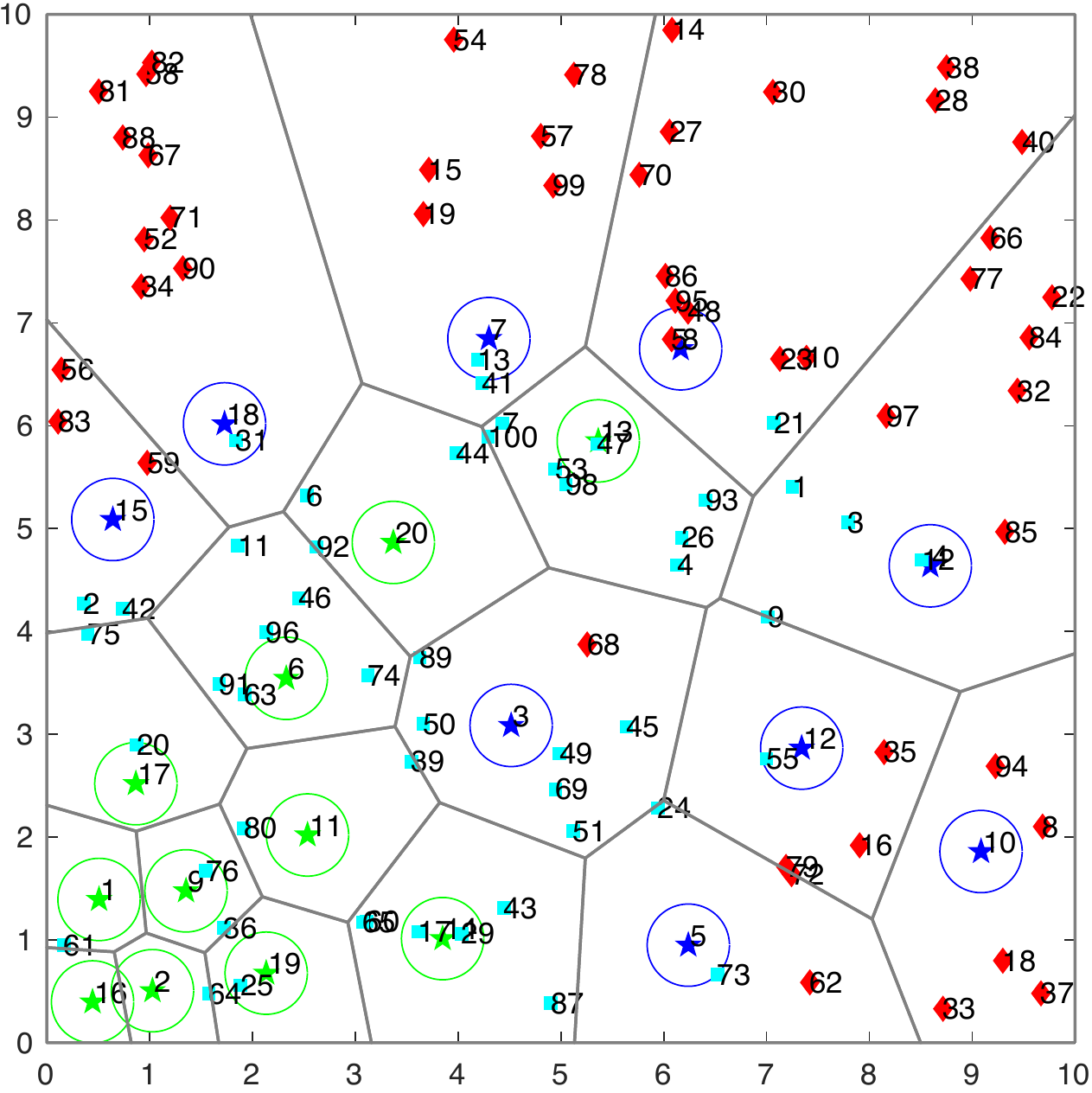}}
%\subfigure[$t=40$]{
 %   \label{fig:t40}
%  \includegraphics[width=0.45\linewidth]{figures/time40_50_100.pdf}}
\subfigure[$t=50$]{
    \label{fig:t50}
  \includegraphics[width=0.45\linewidth]{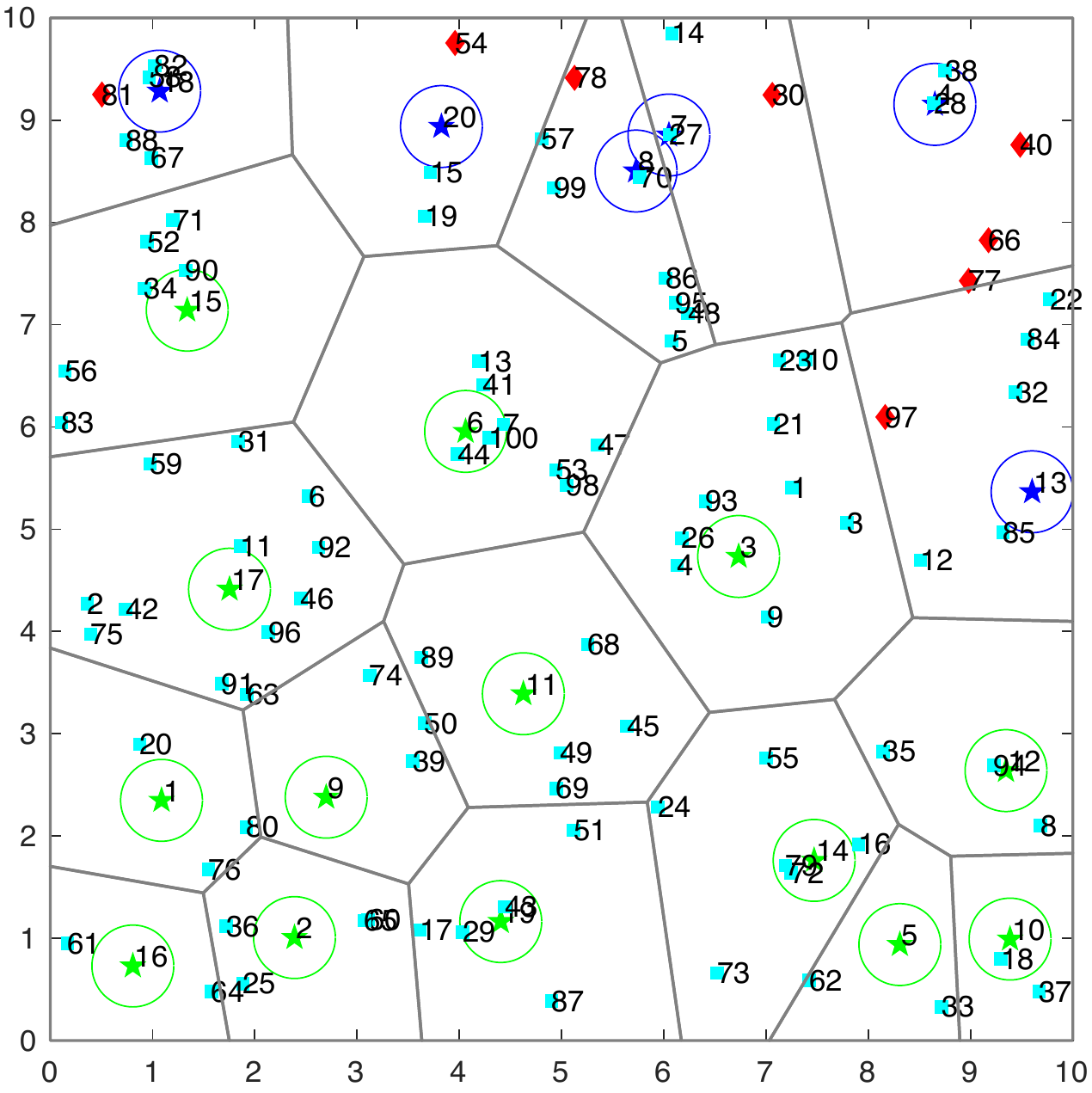}}
 %    \subfigure[$t=65$]{
%   \label{fig:t65}
%  \includegraphics[width=0.45\linewidth]{figures/time65_50_100.pdf}}
     \subfigure[Robot Paths]{
    \label{fig:pathsI}
  \includegraphics[width=0.45\linewidth]{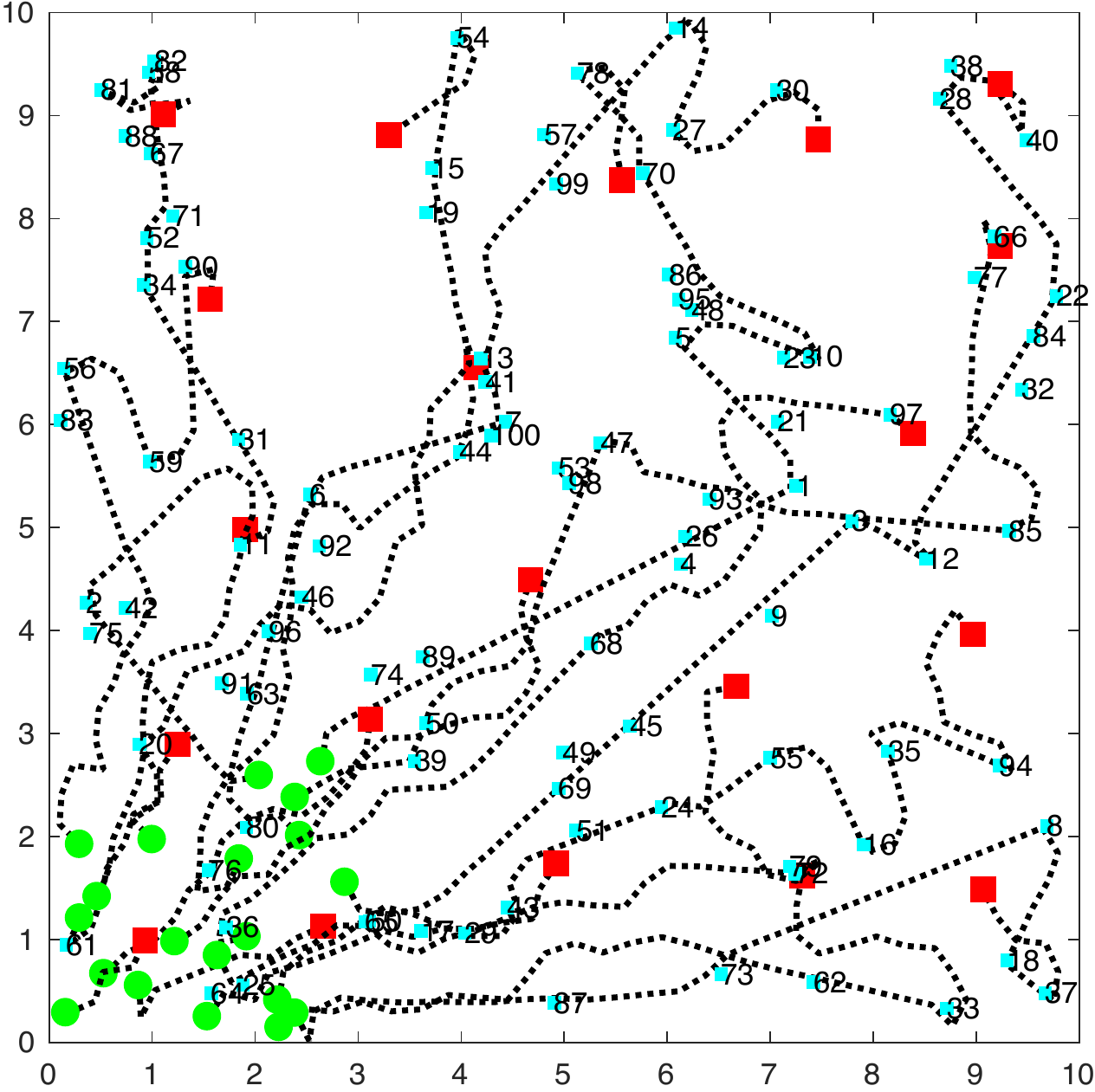}}
%\subfigure[$N=1, M=5$ (infeasible)]{
%    \label{fig:traj15inf}
%\includegraphics[width=0.45\linewidth]{infeasProb_traj.pdf}}
  \caption{Landmark localization scenario: Figures \ref{fig:t0}-\ref{fig:t50} show the configurations of $N=20$ robots at various time instants towards localizing $M=100$ static landmarks. The red diamonds and cyan squares correspond to accurately localized and non-localized landmarks, respectively. The blue and green stars correspond to robots that navigate for information gathering and area coverage purposes, respectively. The circle centered at each robot position illustrates the sensing range while the gray segments depict the Voronoi cells. Figure \ref{fig:pathsI} shows the robot paths where the green and red squares denote the initial and final robot positions.}
  \label{fig:caseI}
\end{figure}

\begin{figure}[t]
  \centering
%\subfigure[$N/M=1/5$, Cost = $34.47$, $F=374$,  Runtime = 23.74secs]{
%\label{fig:traj15f}
%\includegraphics[width=0.45\linewidth]{figures/traj_1vs5_fast.pdf}}
%
\subfigure[Deployment 1]{
    \label{fig:depl1}
  \includegraphics[width=0.45\linewidth]{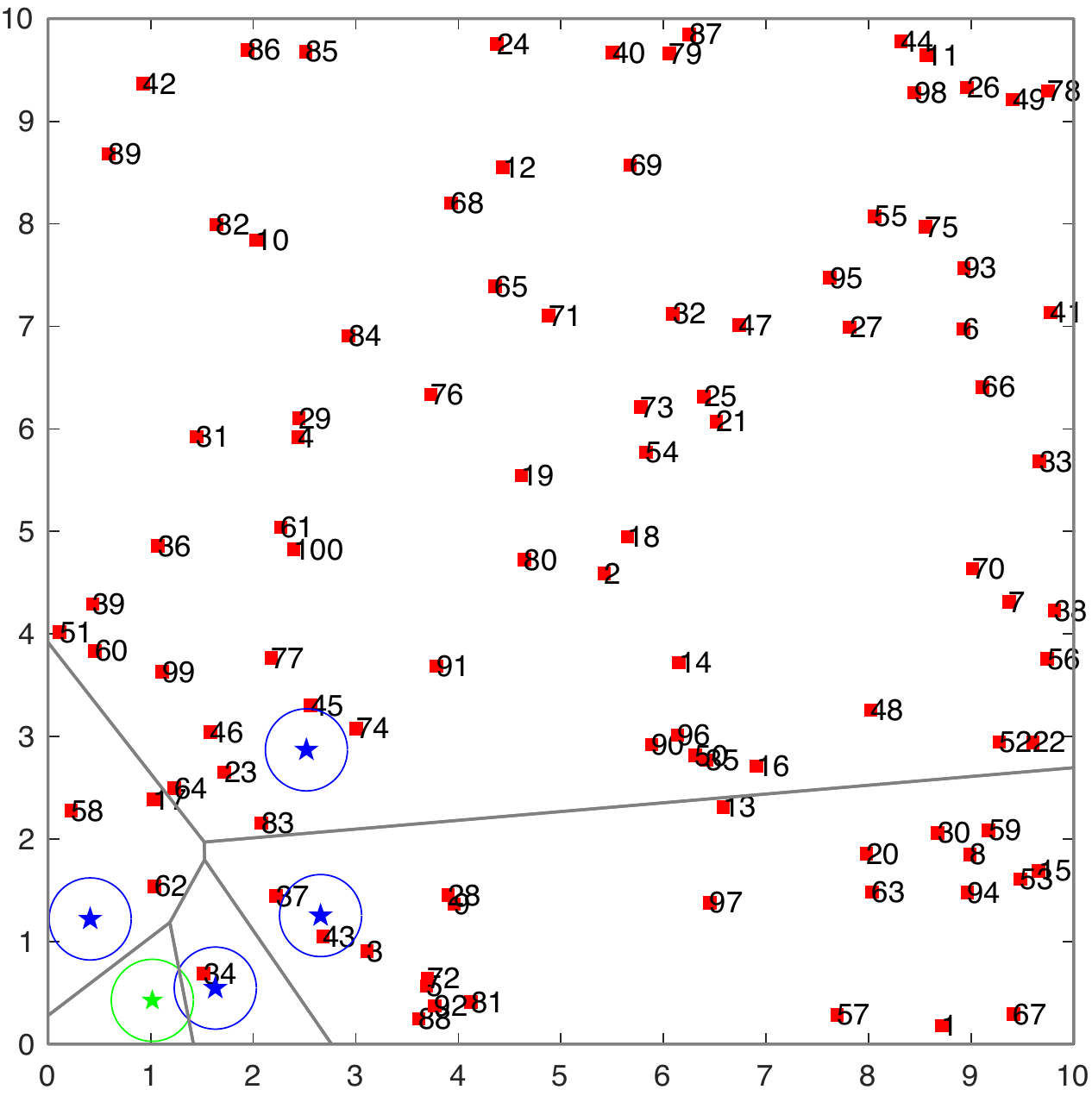}}
 %
% \subfigure[$N/M=1/5$, Cost = $75.03$, $F=374$,  Runtime = 4.19 mins]{
%    \label{fig:biasedGreedy15}
%  \includegraphics[width=0.45\linewidth]{figures/biasedGreedy.pdf}}
   \subfigure[Deployment 2]{
    \label{fig:depl2}
  \includegraphics[width=0.45\linewidth]{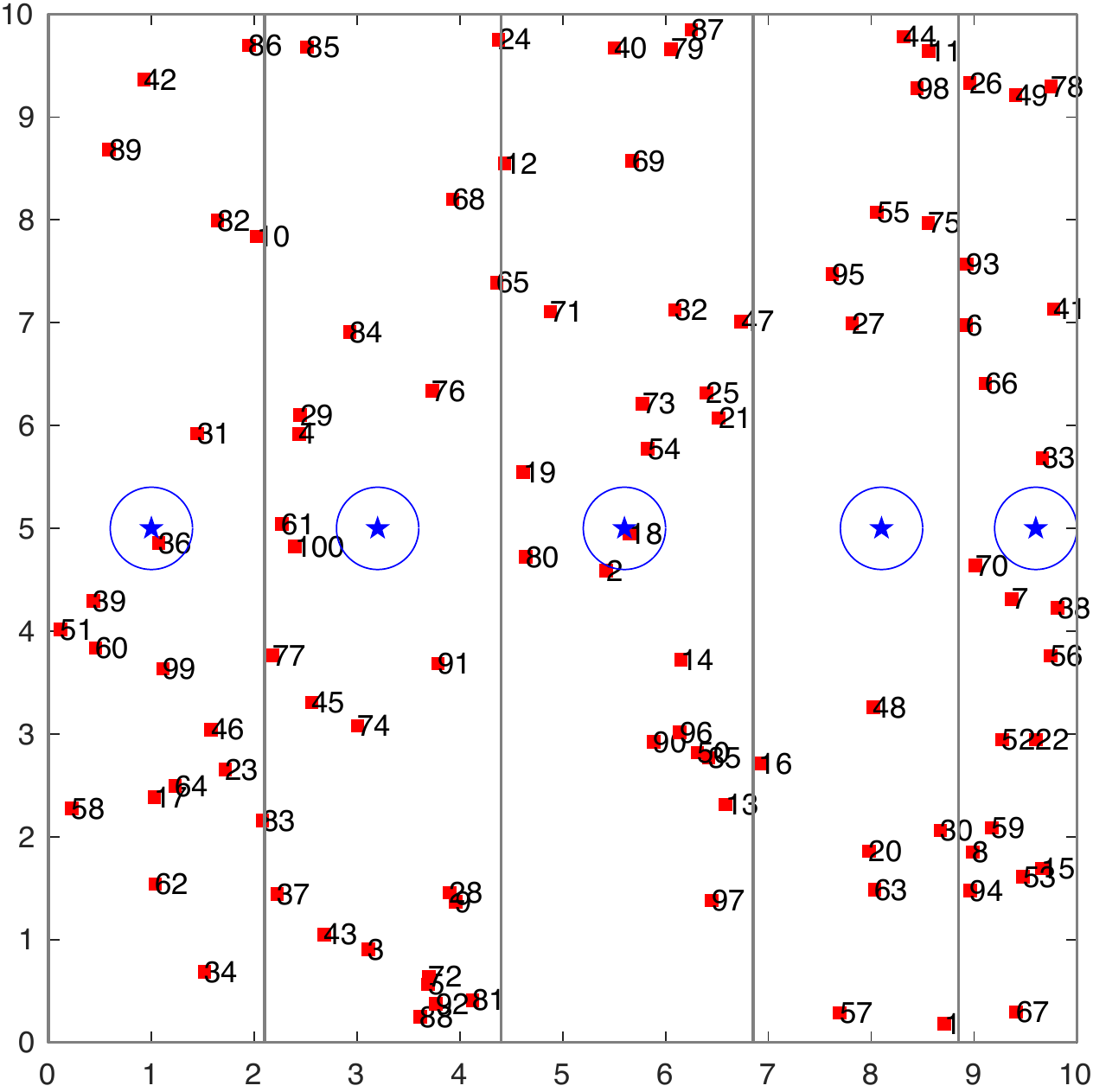}}
 \caption{Graphical depiction of initial multi-robot deployments in Table \ref{ta:vor}. The third initial deployment in Table \ref{ta:vor} differs from the second one only in that the y-coordinate of all robot positions is is $0.1$m.}
  \label{fig:configInit}
\end{figure}

%\subsection{Scalability Analysis}\label{sec:scale}
{\bf{Scalability Analysis:}} First, we examine the scalability performance of  Algorithm \ref{alg:AIA} with respect to the number of robots and the number of landmarks. The results are summarized in Table \ref{tab:scale}. In all case studies of Table \ref{tab:scale}, the parameter $\delta$ is selected to be $\delta=1.8\times10^{-6}$, for all $i\in\ccalM$, while all robots initially reside in the bottom-left corner of the $10\text{m}\times10\text{m}$ environment shown in Figure \ref{fig:caseI}. % and  $\ccalP_j^{\text{goal}}$ is selected to be $\ccalP_j^{\text{goal}}=\Omega_{\text{free}}$, for all $j\in\ccalN$.
In Table \ref{tab:scale} the first line corresponds to the number of robots and landmarks. In the second line, $T_{\text{plan}}$ refers to the average runtime in seconds required for a robot to compute a control action regardless of their role in the workspace (see lines \ref{aia:cntrl1} and \ref{aia:cntrl2} in Alg. \ref{alg:AIA}). Similarly, $T_{\text{vor}}$ shows the average runtime in seconds to compute the Voronoi cells per iteration of Algorithm \ref{alg:AIA} (see line \ref{aia:vor} in Alg. \ref{alg:AIA}). Note that in our implementation, the Voronoi tessellation has been computed in a centralized fashion and, therefore, the corresponding runtime increases as the number of robots increases. Nevertheless, this runtime is always quite small and, therefore, the Voronoi partitioning does not compromise scalability. In the third and fourth row, $T_{\text{total}}$ and $F$ refer to the mean total runtime in minutes and total number of discrete time instants of $10$ simulations, along with the respective standard deviation, required to find the first feasible solution. Note that Algorithm \ref{alg:AIA} has not been implemented in parallel across the robots; instead, in our implementation the robots make local  but sequential decisions, which explains the increasing runtime $T_{\text{total}}$ as the number $N$ of robots increases. 
Observe in Table \ref{tab:scale} that for a given number of landmarks, increasing the number of robots tends to decrease the terminal horizon. Note that the terminal horizon depends on the initial robot configuration. For instance, assuming $N$ robots that are initially uniformly distributed and $M=100$ landmarks, the terminal horizon when $N=20$ and $N=100$ is $35\pm3$ and $15\pm2$, respectively.
Finally, observe that the proposed algorithm can solve estimation tasks that involve hundreds of robots and landmarks, a task that is particularly challenging for existing methods; see e.g., the numerical experiments in \cite{kantaros2019asymptotically}. Figure \ref{fig:caseI} shows the paths that $20$ robots followed to localize $M=100$ landmarks. Observe in this figure that the robots dynamically switch roles in the environment. 

%Observe in Table \ref{tab:scale} that Algorithm \ref{alg:RRT} can design feasible paths very fast even for large number of robots and targets regardless of the robot dynamics; see also Figure \ref{fig:traj1020}. Finally, we also applied Algorithm \ref{alg:RRT} to a scenario where a team of $N=7$ differential drive robots should localize and track $M=20$ targets in a significantly larger workspace, such as a residential area, with dimensions $500\text{m}\times1000\text{m}$. In this scenario, the sensing range of the robots is $20\text{m}$, and the motion primitives are selected as $u\in\{0,2\}\text{m/s}$ and $\omega\in\set{0,\pm \pi/4,\pm \pi/2, \pm \pi/1.33, \pm \pi}\text{rad/s}$. Algorithm \ref{alg:RRT} generated robot paths in $12.23$ \text{mins} with terminal horizon $F=3769$ that are depicted in Figure \ref{fig:largeSpace}.

%\subsection{Online vs Offline Planning}\label{sec:vor}
{\bf{Online vs Offline Planning:}} Second, we examine the effect of dynamically updating the Voronoi partitioning - and, consequently, the landmarks assigned to the robots - and accordingly re-planning. The results are summarized in Table \ref{ta:vor} and pertain to a case study with $N=5$ and $M=100$. In Table \ref{ta:vor}, the first row corresponds to three different initial configurations illustrated in Figure \ref{fig:configInit}.  The second and the third row show the mean terminal planning horizon $F_{\text{offline}}$ and $F_{\text{online}}$ of $10$ simulations required to localize all landmarks without and with re-planning, respectively. %(i.e., all robots follow the paths designed offline at time $t=0$) and with re-planning as per Algorithm 1.
%fixed Voronoi cells (as computed at $t=0$) and dynamic Voronoi cells that are continuously updated, respectively.
Observe that the total time is always larger when offline planning is considered. This becomes more pronounced in configuration 1 (Fig. \ref{fig:depl1}), as there are robots that have to stay within a small region the whole time which is not the case when online re-planning is considered due to its adaptive nature. As a result, in an online setting the burden of localizing all $M$ landmarks is shared among all robots resulting in accomplishing the AIA task sooner.

{\bf{Effect of Communication:}} Third, we evaluate the effect of all-time and all-to-all communication among the robots on the estimation performance during online planning. Specifically, here we consider a scenario with $N=50$ robots and $M=100$ static landmarks and we assume that communication occurs intermittently and periodically, i.e., every $T>0$ time units. We observed that increasing the period $T$ results in longer terminal planning horizons $F$, as expected, since the robots plan paths without always having access to global estimates of the hidden state. In particular, for $T=1$ (all-time communication), $T=2$, $T=10$, and $T=15$, the average planning horizon of $10$ simulation studies was $63$, $70$, $125$, and $150$, respectively.

%In this section, we examine the effect of dynamically updating the Voronoi partitioning - and, consequently, the landmarks assigned to the robots - and accordingly re-planning. The results are summarized in Table \ref{ta:vor} where the first row corresponds to the number of robots and landmarks. The second and third row show the total number of discrete time instants $F_{\text{offline}}$ and $F_{\text{online}}$ required to localize all landmarks with no re-planning (i.e., all robots follow the paths designed offline at time $t=0$) and with re-planning.
%fixed Voronoi cells (as computed at $t=0$) and dynamic Voronoi cells that are continuously updated, respectively.
%Observe that the total time is always smaller when online re-planning is considered. The reason is that in the latter case as the robots - that are not responsible for any landmark - spread in the environment to maximize area coverage, they approach non-localized landmarks that fall within their region of responsibility. As a result, the burden of localizing all $M$ landmarks is shared among all robots resulting in accomplishing the AIA sooner. %Recall that updating the Voronoi cells may sacrifice completeness of Algorithm \ref

\begin{table}[]
\caption{Online vs Offline Planning}
\label{ta:vor}
\centering
\begin{tabular}{|l|l|l|l|}
\hline
N=5, M=100     & Deployment 1 & Deployment 2 & Deployment 3 \\ \hline
$F_{\text{offline}}$ & 782 $\pm$  44    & 190 $\pm$ 30     & 210 $\pm$ 30       \\ \hline
$F_{\text{online}}$  & 141 $\pm$ 16     & 137 $\pm$ 8      & 146 $\pm$ 26     \\ \hline
\end{tabular}
\end{table}

%\subsection{Extensions to Mobile Landmarks}\label{sec:mobile}

{\bf{AIA with Drones:}} Algorithm \ref{alg:AIA} generates paths that respect the robot dynamics as captured in \eqref{constr34}. A common limitation of sampling-based algorithms is that the more complex the robot dynamics is, the longer it takes to generate feasible paths; see e.g., \cite{kantaros2019asymptotically}. To mitigate this issue, an approach that we investigate in this case study is to generate paths for simple robot dynamics that need to be followed by robots with more complex dynamics. Particularly, in this section, we present simulation studies that involve a team of $N=30$ AsTech Firefly Unmanned Aerial Vehicles (UAVs) with range limited field of view equal to $10$m that operate in an environment with dimensions $200\text{m}\times 200 \text{m}$ occupied by $M=50$ static landmarks; see Figure \ref{fig:drones}.

\begin{figure}[t]
  \centering
    \subfigure[$t=0$]{
    \label{fig:t01}
    \includegraphics[width=0.48\linewidth]{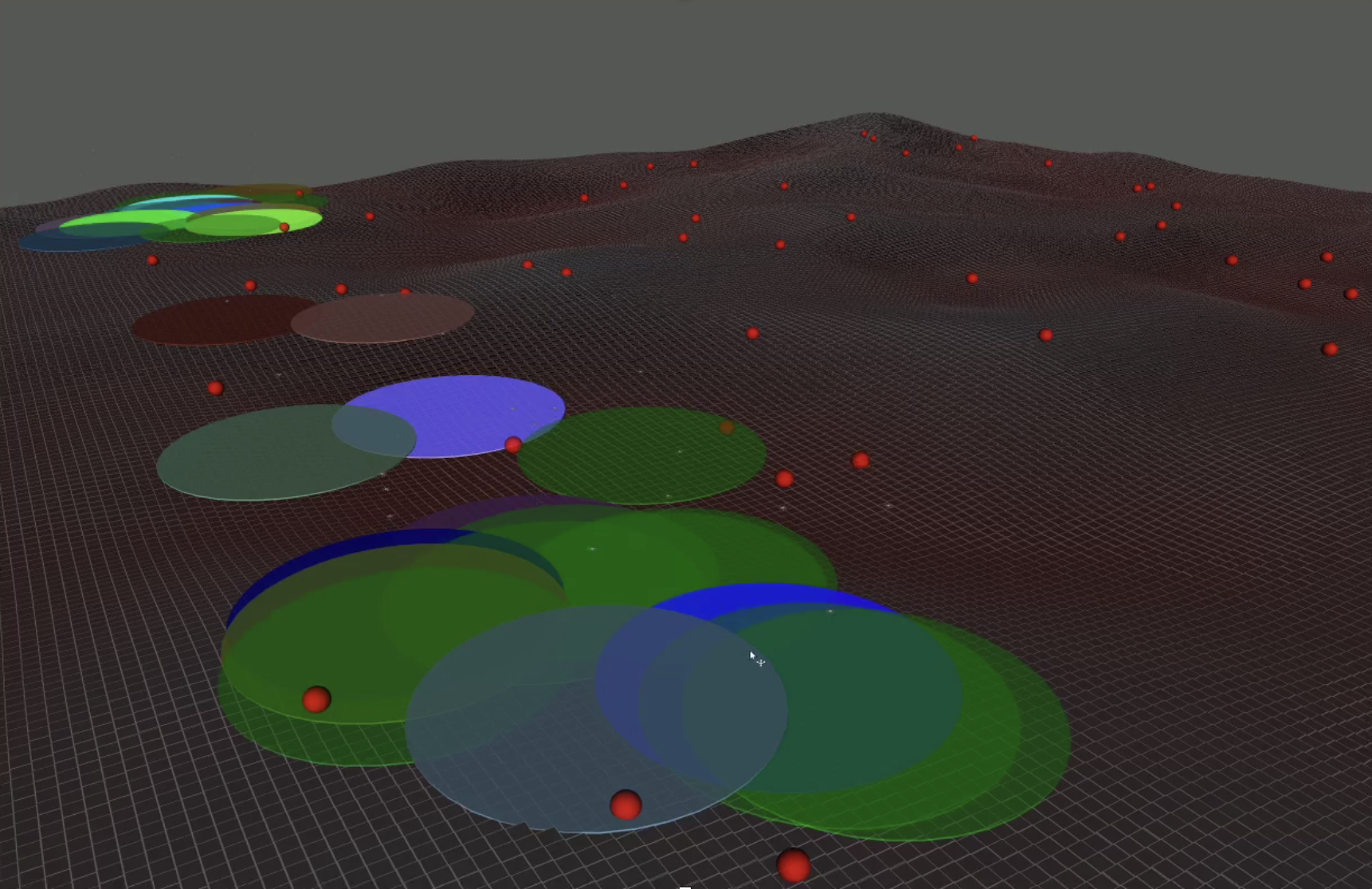}}
     \subfigure[]{
    \label{fig:t21}
  \includegraphics[width=0.48\linewidth]{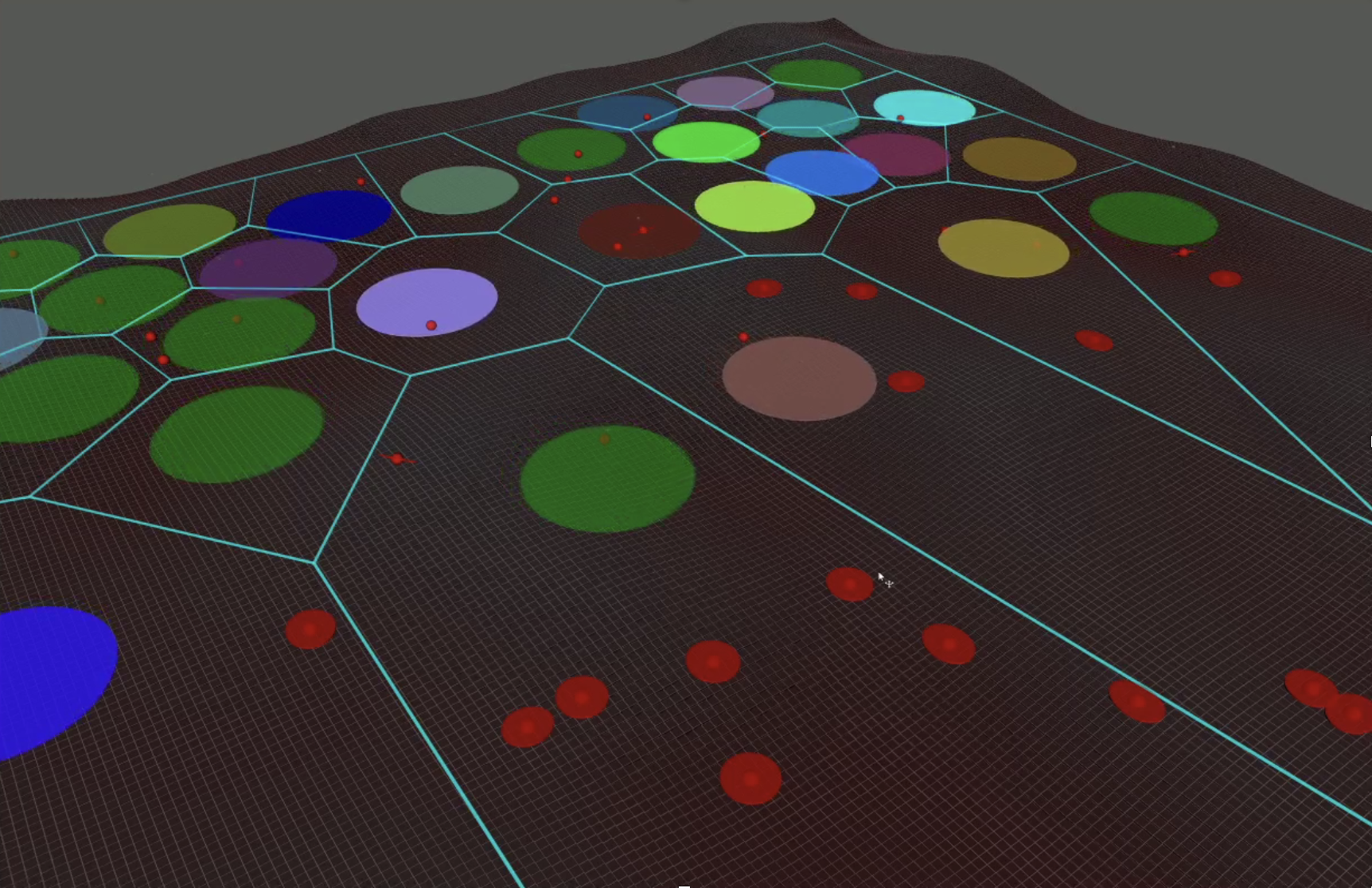}}
       \subfigure[]{
    \label{fig:t31}
  \includegraphics[width=0.48\linewidth]{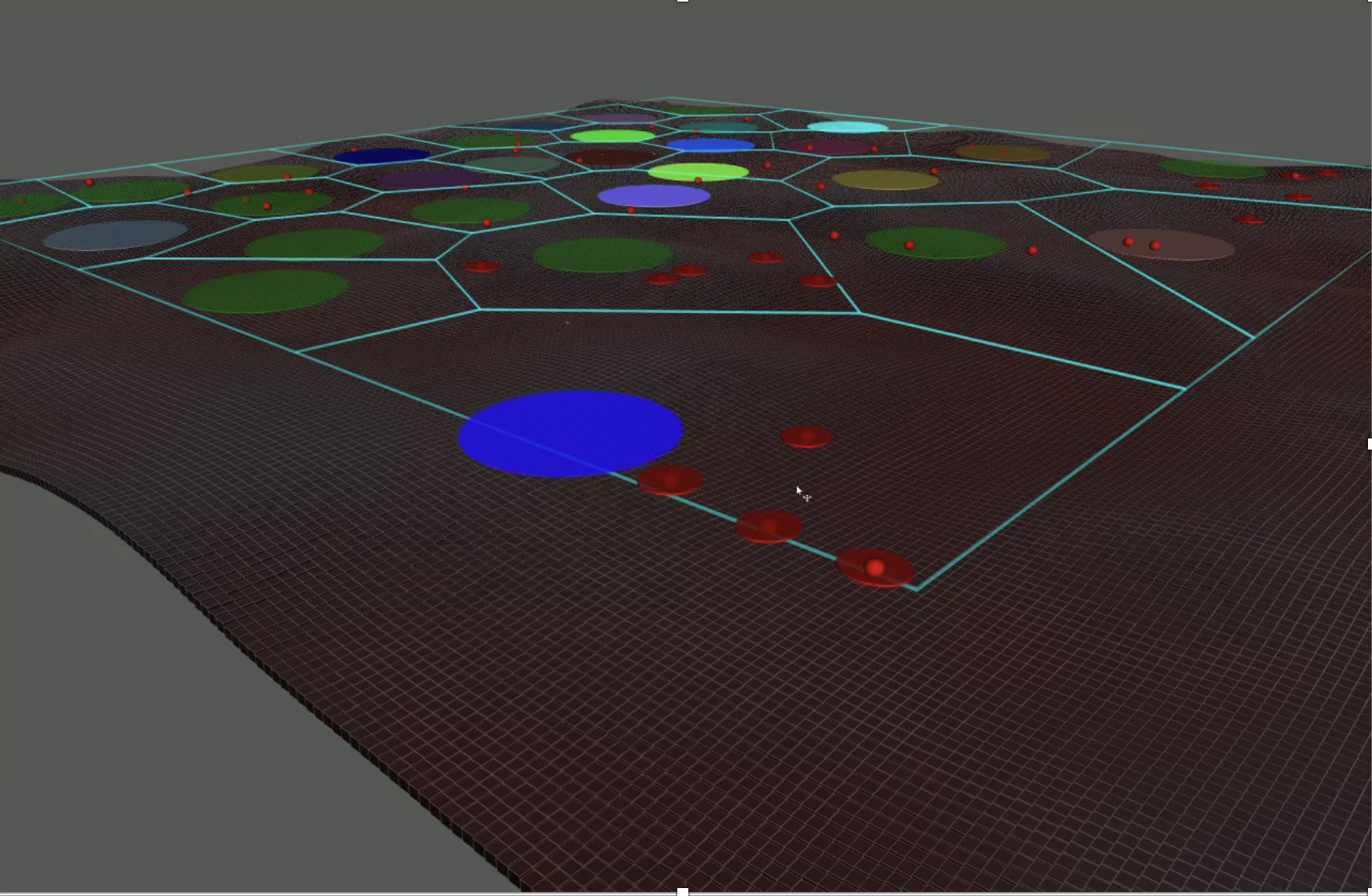}}
 %      \subfigure[$t=73$]{
 %   \label{fig:t73d}
%  \includegraphics[width=0.45\linewidth]{figures/time73_50_100_dyn.pdf}}
     \subfigure[]{
    \label{fig:t41}
  \includegraphics[width=0.48\linewidth]{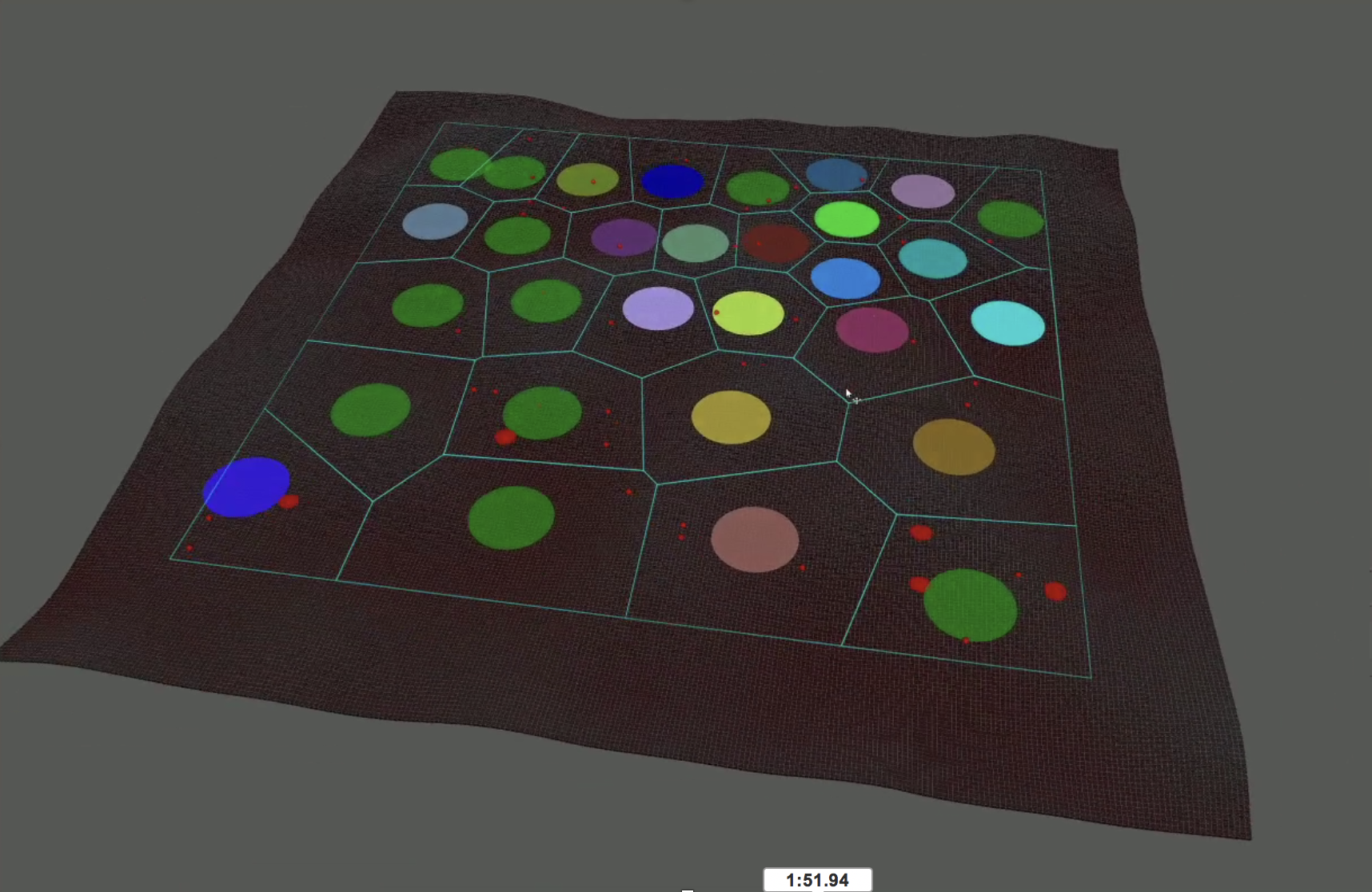}}
%\subfigure[$N=1, M=5$ (infeasible)]{
%    \label{fig:traj15inf}
%\includegraphics[width=0.45\linewidth]{infeasProb_traj.pdf}}
 \caption{Landmark localization scenario: Figures \ref{fig:t01}-\ref{fig:t41} show successive snapshots of a team of $N=30$ UAVs with limited field of view (colored disks) towards localizing $M=50$ landmarks (red spheres). %The red diamonds and cyan squares correspond to accurately localized and non-localized targets, respectively. The blue and green stars correspond to robots that navigate for information gathering and area coverage purposes, respectively. The circle centered at each robot position illustrate the sensing range while the gray segments depict the Voronoi cells. 
 The red ellipses denote the uncertainty about the landmark positions. The full video can be found in \cite{sim_AIA}.}
  \label{fig:drones}
\end{figure}

The initial configuration of the UAVs and the landmarks are shown in Figure \ref{fig:t01}. Given the paths generated by Algorithm \ref{alg:AIA} considering differential drive dynamics, we compute minimum jerk trajectories, defined by fifth-order polynomials, that connect consecutive waypoints in the nominal paths. To determine the coefficients of this polynomial, we impose boundary conditions on the UAV positions that require the travel time between consecutive waypoints in the nominal paths to be $T=2$ seconds for all UAVs \cite{lavalle2006planning}. %; see Figure \ref{fig:sync}. 
%Note that the minimum jerk trajectories are smooth enough for the UAVs to follow.  Observe in Figure \ref{fig:paths} that deviation between the nominal paths generated by Algorithm \ref{alg:RRT} and the minimum jerk paths is quite small. %and, in fact, the latter paths still satisfy the assigned specification.
%
The UAVs are controlled to follow the resulting trajectories using the ROS package developed in \cite{Furrer2016}. Snapshots of the UAVs navigating the environment to localize all landmarks are presented in Figure \ref{fig:drones}.

{\bf{Mobile Landmarks:}} Finally, we illustrate the performance of Algorithm \ref{alg:AIA} in an environment with mobile landmarks. Specifically, we consider $M=100$ landmarks, among which $8$ are governed by noisy linear time invariant dynamics in the form of $\bbx_i(t+1)=\bbA_i\bbx_i(t)+\bbB_i\bba_i(t)+\bbw_i(t)$ and the rest are static. The trajectories of all landmarks and robots are illustrated in Figure \ref{fig:caseId}. Observe in this figure that because of the dynamic nature of the landmarks, there are landmarks that need to be revisited to decrease their uncertainty giving rise to target tracking behaviors; see e.g., landmark $95$ in Figures \ref{fig:t50d} and \ref{fig:t60d}. We note that although Algorithm \ref{alg:AIA} is not complete in the presence of dynamic hidden states, it was able to find feasible controllers for a large team of robots and landmarks. %Also, note that in this case study a Voronoi decomposition of the environment may not result in a `fair' decomposition of the global AIA task across the robots as it does not account the 
\begin{figure}[t]
  \centering
%\subfigure[$N/M=1/5$, Cost = $34.47$, $F=374$,  Runtime = 23.74secs]{
%\label{fig:traj15f}
%\includegraphics[width=0.45\linewidth]{figures/traj_1vs5_fast.pdf}}
%
\subfigure[$t=0$]{
    \label{fig:t0d}
  \includegraphics[width=0.45\linewidth]{figures/time1_50_100.pdf}}
 %
% \subfigure[$N/M=1/5$, Cost = $75.03$, $F=374$,  Runtime = 4.19 mins]{
%    \label{fig:biasedGreedy15}
%  \includegraphics[width=0.45\linewidth]{figures/biasedGreedy.pdf}}
%   \subfigure[$t=30$]{
%    \label{fig:t30d}
%  \includegraphics[width=0.45\linewidth]{figures/time40_50_100_dyn.pdf}}
     \subfigure[$t=50$]{
    \label{fig:t50d}
  \includegraphics[width=0.45\linewidth]{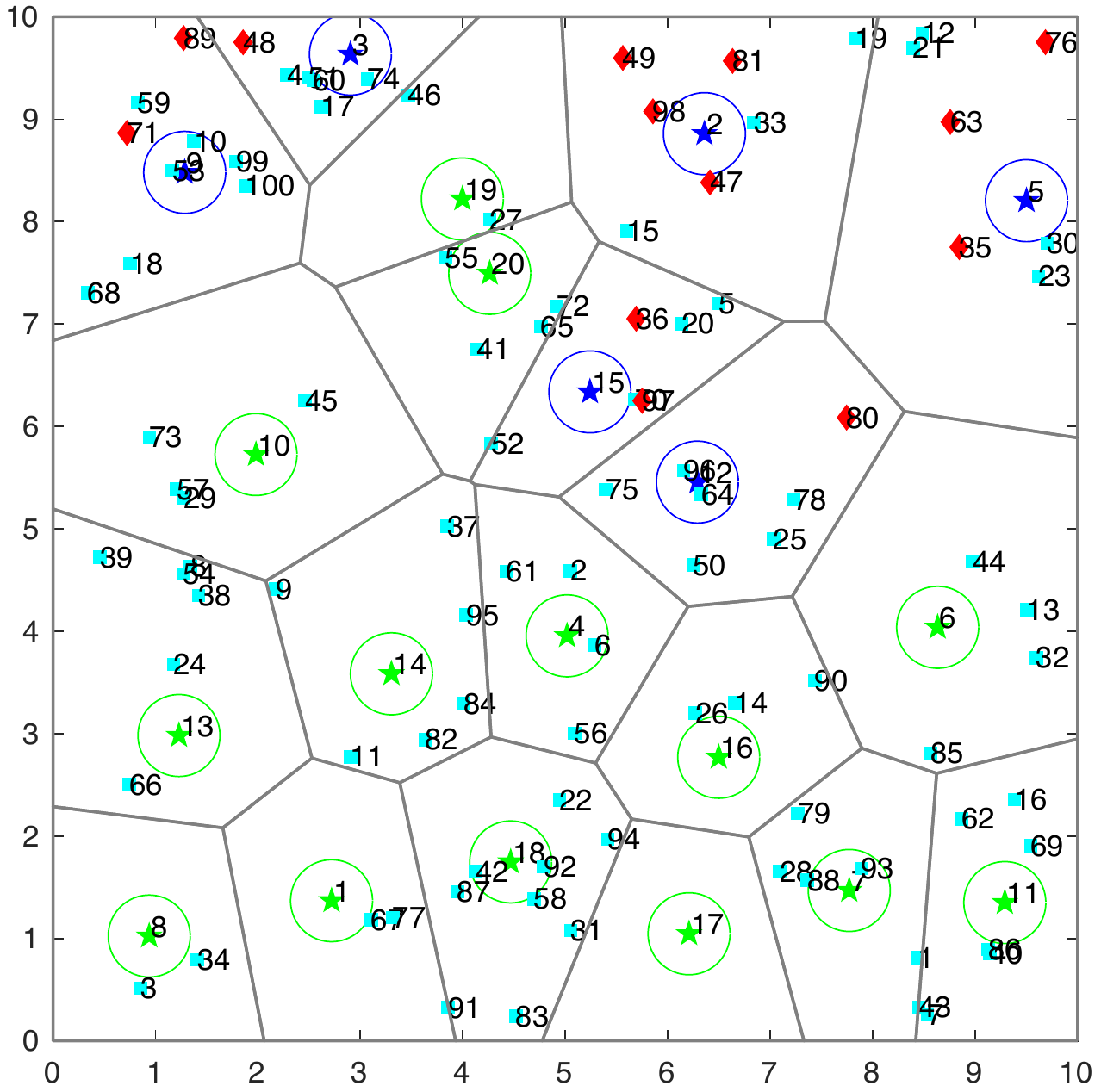}}
       \subfigure[$t=60$]{
    \label{fig:t60d}
  \includegraphics[width=0.45\linewidth]{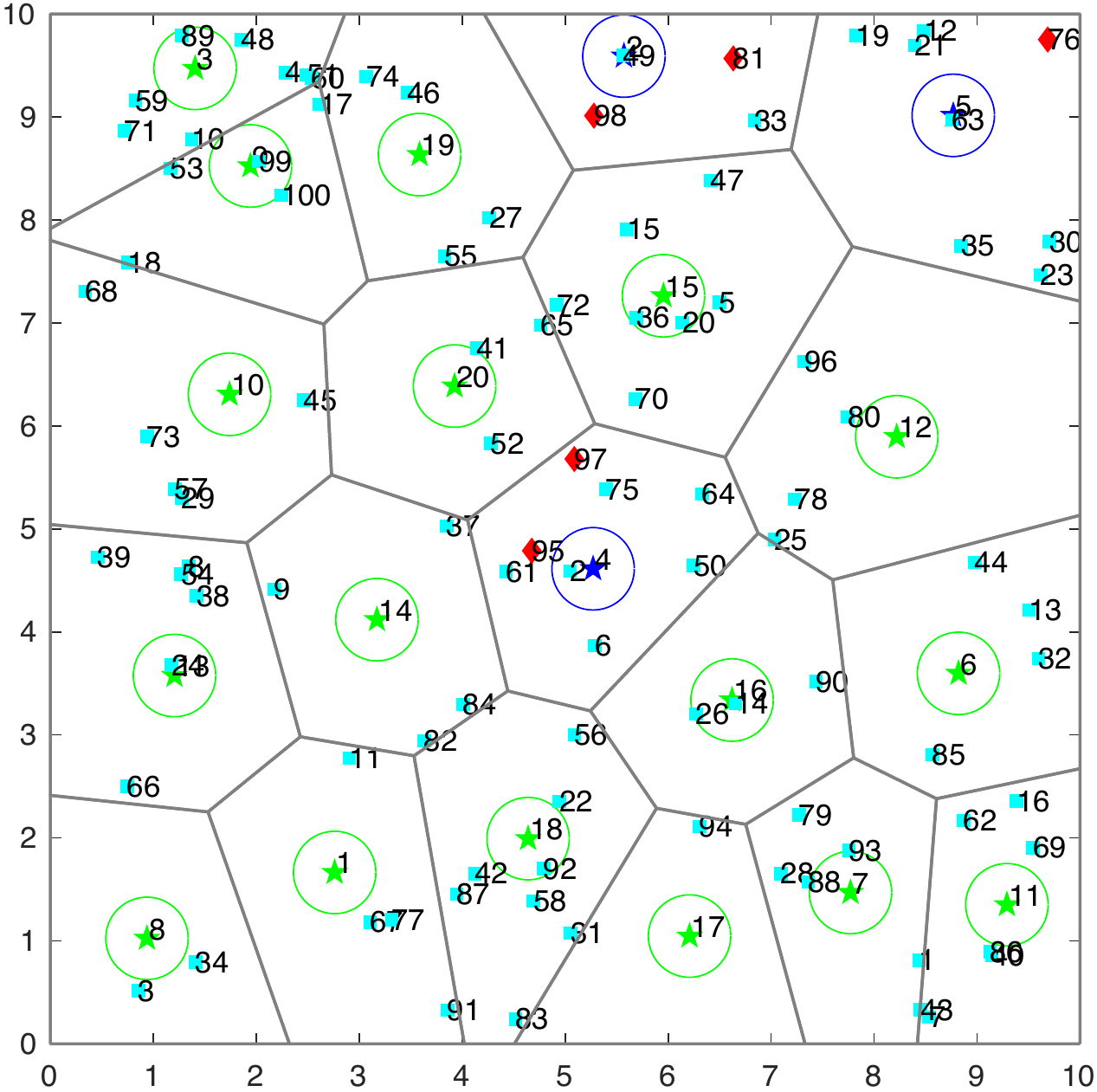}}
 %      \subfigure[$t=73$]{
 %   \label{fig:t73d}
%  \includegraphics[width=0.45\linewidth]{figures/time73_50_100_dyn.pdf}}
     \subfigure[Robot \& Landmark Paths]{
    \label{fig:pathsId}
  \includegraphics[width=0.45\linewidth]{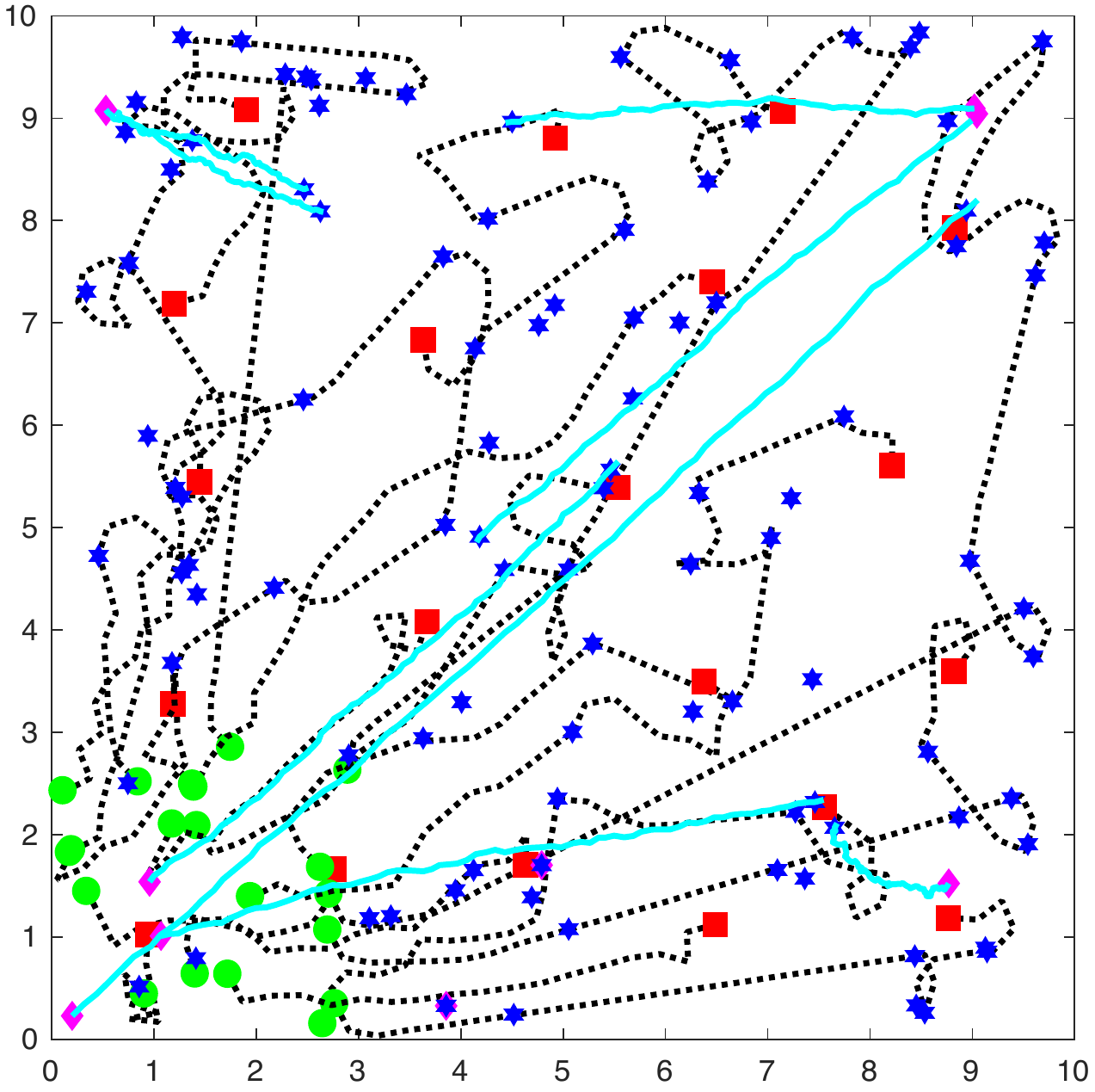}}
%\subfigure[$N=1, M=5$ (infeasible)]{
%    \label{fig:traj15inf}
%\includegraphics[width=0.45\linewidth]{infeasProb_traj.pdf}}
 \caption{Target tracking and localization scenario: Figures \ref{fig:t0d}-\ref{fig:t60d} show the configurations of $N=20$ robots at various time instants towards localizing $M=100$ landmarks among which $8$ are mobile and the rest are static. %The red diamonds and cyan squares correspond to accurately localized and non-localized targets, respectively. The blue and green stars correspond to robots that navigate for information gathering and area coverage purposes, respectively. The circle centered at each robot position illustrate the sensing range while the gray segments depict the Voronoi cells. 
 Figure \ref{fig:pathsId} shows the robot and target paths. The green circles and red squares denote the initial and final robot positions while the magenta diamonds and blue stars denote the initial and final landmark locations. Static landmarks are shown with blue stars.}
  \label{fig:caseId}
\end{figure}
%\subsection{Unknown Number of Landmarks}
%\textcolor{red}{If time and space permit, we can add this subsection...}

%\subsection{Gazebo Simulations(?)}\label{sec:gaz}
%\textcolor{red}{here we need to add a couple of gazebo simulations }

%\textcolor{red}{here we need to add simulations related to at least one the 'extensions': unknown landmarks or mobile targets, or both?}

% ------------------------------------------------------------------------------------------------------------------------------ %
\section{Conclusion} \label{sec:Concl}
In this paper, we proposed a new highly scalable, non-myopic, and probabilistically complete planning algorithm for multi-robot AIA tasks. Extensive simulation studies validated the theoretical analysis and showed that the proposed method can quickly compute sensor policies that satisfy desired uncertainty thresholds for large-scale AIA tasks. % which was impossible using existing centralized methods.

\appendices
%\section{Proof Completeness, Optimality, \& Complexity}\label{sec:prop}
%\input{files/proofs}

%\section{Incremental Construction of Trees}\label{appB}
%\input{files/tree}

\bibliographystyle{IEEEtran}
\bibliography{YK_bib.bib}

\begin{thebibliography}{10}
\providecommand{\url}[1]{#1}
\csname url@samestyle\endcsname
\providecommand{\newblock}{\relax}
\providecommand{\bibinfo}[2]{#2}
\providecommand{\BIBentrySTDinterwordspacing}{\spaceskip=0pt\relax}
\providecommand{\BIBentryALTinterwordstretchfactor}{4}
\providecommand{\BIBentryALTinterwordspacing}{\spaceskip=\fontdimen2\font plus
\BIBentryALTinterwordstretchfactor\fontdimen3\font minus
  \fontdimen4\font\relax}
\providecommand{\BIBforeignlanguage}[2]{{%
\expandafter\ifx\csname l@#1\endcsname\relax
\typeout{** WARNING: IEEEtran.bst: No hyphenation pattern has been}%
\typeout{** loaded for the language `#1'. Using the pattern for}%
\typeout{** the default language instead.}%
\else
\language=\csname l@#1\endcsname
\fi
#2}}
\providecommand{\BIBdecl}{\relax}
\BIBdecl

\bibitem{huang2015bank}
G.~Huang, K.~Zhou, N.~Trawny, and S.~I. Roumeliotis, ``A bank of maximum a
  posteriori (map) estimators for target tracking,'' \emph{IEEE Transactions on
  Robotics}, vol.~31, no.~1, pp. 85--103, 2015.

\bibitem{lu2018mobile}
Q.~Lu and Q.-L. Han, ``Mobile robot networks for environmental monitoring: A
  cooperative receding horizon temporal logic control approach,'' \emph{IEEE
  Transactions on Cybernetics}, 2018.

\bibitem{carlone2014active}
L.~Carlone, J.~Du, M.~K. Ng, B.~Bona, and M.~Indri, ``Active {SLAM} and
  exploration with particle filters using {K}ullback-{L}eibler divergence,''
  \emph{Journal of Intelligent \& Robotic Systems}, vol.~75, no.~2, pp.
  291--311, 2014.

\bibitem{atanasov2015distributed}
N.~A. Atanasov, J.~Le~Ny, and G.~J. Pappas, ``Distributed algorithms for
  stochastic source seeking with mobile robot networks,'' \emph{Journal of
  Dynamic Systems, Measurement, and Control}, vol. 137, no.~3, p. 031004, 2015.

\bibitem{kumar2004robot}
V.~Kumar, D.~Rus, and S.~Singh, ``Robot and sensor networks for first
  responders,'' \emph{IEEE Pervasive computing}, vol.~3, no.~4, pp. 24--33,
  2004.

\bibitem{kantaros2019asymptotically}
Y.~Kantaros, B.~Schlotfeldt, N.~Atanasov, and G.~J. Pappas, ``Asymptotically
  optimal planning for non-myopic multi-robot information gathering.'' in
  \emph{Robotics: Science and Systems}, 2019.

\bibitem{Cortes_TRA04}
J.~Cort\'{e}s, S.~Martinez, T.~Karatas, and F.~Bullo, ``Coverage control for
  mobile sensing networks,'' \emph{IEEE Transactions on Robotics and
  Automation}, vol.~20, no.~2, pp. 243--255, 2004.

\bibitem{kantaros2015distributed}
Y.~Kantaros, M.~Thanou, and A.~Tzes, ``Distributed coverage control for concave
  areas by a heterogeneous robot--swarm with visibility sensing constraints,''
  \emph{Automatica}, vol.~53, pp. 195--207, 2015.

\bibitem{leung2006active}
C.~Leung, S.~Huang, and G.~Dissanayake, ``Active slam using model predictive
  control and attractor based exploration,'' in \emph{2006 IEEE/RSJ
  International Conference on Intelligent Robots and Systems}, 2006, pp.
  5026--5031.

\bibitem{smith2018distributed}
A.~J. Smith and G.~A. Hollinger, ``Distributed inference-based multi-robot
  exploration,'' \emph{Autonomous Robots}, vol.~42, no.~8, pp. 1651--1668,
  2018.

\bibitem{corah2019communication}
M.~Corah, C.~O’Meadhra, K.~Goel, and N.~Michael, ``Communication-efficient
  planning and mapping for multi-robot exploration in large environments,''
  \emph{IEEE Robotics and Automation Letters}, vol.~4, no.~2, pp. 1715--1721,
  2019.

\bibitem{wang2020autonomous}
J.~Wang and B.~Englot, ``Autonomous exploration with
  expectation-maximization,'' in \emph{Robotics Research}.\hskip 1em plus 0.5em
  minus 0.4em\relax Springer, 2020, pp. 759--774.

\bibitem{sim_AIA}
Simulation\_Video, \texttt{\url{https://vimeo.com/468212369}}.

\bibitem{martinez2006optimal}
S.~Mart{\'\i}nez and F.~Bullo, ``Optimal sensor placement and motion
  coordination for target tracking,'' \emph{Automatica}, vol.~42, no.~4, pp.
  661--668, 2006.

\bibitem{graham2008cooperative}
R.~Graham and J.~Cort{\'e}s, ``A cooperative deployment strategy for optimal
  sampling in spatiotemporal estimation,'' in \emph{47th IEEE Conference on
  Decision and Control}, 2008, pp. 2432--2437.

\bibitem{dames2012decentralized}
P.~Dames, M.~Schwager, V.~Kumar, and D.~Rus, ``A decentralized control policy
  for adaptive information gathering in hazardous environments,'' in
  \emph{Decision and Control (CDC), 2012 IEEE 51st Annual Conference on}.\hskip
  1em plus 0.5em minus 0.4em\relax IEEE, 2012, pp. 2807--2813.

\bibitem{charrow2014approximate}
B.~Charrow, V.~Kumar, and N.~Michael, ``Approximate representations for
  multi-robot control policies that maximize mutual information,''
  \emph{Autonomous Robots}, vol.~37, no.~4, pp. 383--400, 2014.

\bibitem{meyer2015distributed}
F.~Meyer, H.~Wymeersch, M.~Fr{\"o}hle, and F.~Hlawatsch, ``Distributed
  estimation with information-seeking control in agent networks,'' \emph{IEEE
  Journal on Selected Areas in Communications}, vol.~33, no.~11, pp.
  2439--2456, 2015.

\bibitem{corah2018distributed}
M.~Corah and N.~Michael, ``Distributed submodular maximization on partition
  matroids for planning on large sensor networks,'' in \emph{IEEE Conference on
  Decision and Control}, Miami Beach, FL, 2018, pp. 6792--6799.

\bibitem{le2009trajectory}
J.~Le~Ny and G.~J. Pappas, ``On trajectory optimization for active sensing in
  gaussian process models,'' in \emph{Proceedings of the 48th IEEE Conference
  on Decision and Control, held jointly with the 28th Chinese Control
  Conference}, Shanghai, China, 2009, pp. 6286--6292.

\bibitem{singh2009efficient}
A.~Singh, A.~Krause, C.~Guestrin, and W.~J. Kaiser, ``Efficient informative
  sensing using multiple robots,'' \emph{Journal of Artificial Intelligence
  Research}, vol.~34, pp. 707--755, 2009.

\bibitem{atanasov2015decentralized}
N.~Atanasov, J.~Le~Ny, K.~Daniilidis, and G.~J. Pappas, ``Decentralized active
  information acquisition: Theory and application to multi-robot {SLAM},'' in
  \emph{IEEE International Conference on Robotics and Automation}, Seattle, WA,
  2015, pp. 4775--4782.

\bibitem{schlotfeldt2018anytime}
B.~Schlotfeldt, D.~Thakur, N.~Atanasov, V.~Kumar, and G.~J. Pappas, ``Anytime
  planning for decentralized multirobot active information gathering,''
  \emph{IEEE Robotics and Automation Letters}, vol.~3, no.~2, pp. 1025--1032,
  2018.

\bibitem{levine2010information}
D.~Levine, B.~Luders, and J.~P. How, ``Information-rich path planning with
  general constraints using rapidly-exploring random trees,'' in \emph{AIAA
  Infotech at Aerospace Conference, Atlanta, GA}, 2010.

\bibitem{hollinger2014sampling}
G.~A. Hollinger and G.~S. Sukhatme, ``Sampling-based robotic information
  gathering algorithms,'' \emph{The International Journal of Robotics
  Research}, vol.~33, no.~9, pp. 1271--1287, 2014.

\bibitem{khodayi2019distributed}
R.~Khodayi-mehr, Y.~Kantaros, and M.~M. Zavlanos, ``Distributed state
  estimation using intermittently connected robot networks,'' \emph{IEEE
  Transactions on Robotics}, vol.~35, no.~3, pp. 709--724, 2019.

\bibitem{lan2016rapidly}
X.~Lan and M.~Schwager, ``Rapidly exploring random cycles: Persistent
  estimation of spatiotemporal fields with multiple sensing robots,''
  \emph{IEEE Transactions on Robotics}, vol.~32, no.~5, pp. 1230--1244, 2016.

\bibitem{kantaros2020semantic}
Y.~Kantaros, Q.~Jin, and G.~Pappas, ``Sensor-based temporal logic planning in
  uncertain semantic maps,'' \emph{\url{https://arxiv.org/pdf/2012.10490.pdf}},
  2020.

\bibitem{rosinol2019kimera}
A.~Rosinol, M.~Abate, Y.~Chang, and L.~Carlone, ``Kimera: an open-source
  library for real-time metric-semantic localization and mapping,'' \emph{arXiv
  preprint arXiv:1910.02490}, 2019.

\bibitem{bowman2017probabilistic}
S.~L. Bowman, N.~Atanasov, K.~Daniilidis, and G.~J. Pappas, ``Probabilistic
  data association for semantic {SLAM},'' in \emph{IEEE International
  Conference on Robotics and Automation}, Singapore, May-June 2017, pp.
  1722--1729.

\bibitem{atanasov2014information}
N.~Atanasov, J.~Le~Ny, K.~Daniilidis, and G.~J. Pappas, ``Information
  acquisition with sensing robots: Algorithms and error bounds,'' in \emph{IEEE
  International Conference on Robotics and Automation}, Hong Kong, China, 2014,
  pp. 6447--6454.

\bibitem{yamauchi1997frontier}
B.~Yamauchi, ``A frontier-based approach for autonomous exploration,'' in
  \emph{Proceedings 1997 IEEE International Symposium on Computational
  Intelligence in Robotics and Automation CIRA'97.'Towards New Computational
  Principles for Robotics and Automation'}, 1997, pp. 146--151.

\bibitem{Cortes_ESAIMCOCV05}
J.~Cort\'{e}s, S.~Martinez, and F.~Bullo, ``Spatially-distributed coverage
  optimization and control with limited-range interactions,'' \emph{ESAIM:
  Control, Optimisation and Calculus of Variations}, vol.~11, no.~4, pp.
  691--719, 2005.

\bibitem{lavalle2006planning}
S.~M. LaValle, \emph{Planning algorithms}.\hskip 1em plus 0.5em minus
  0.4em\relax Cambridge university press, 2006.

\bibitem{Furrer2016}
\BIBentryALTinterwordspacing
F.~Furrer, M.~Burri, M.~Achtelik, and R.~Siegwart, \emph{Robot Operating System
  (ROS): The Complete Reference (Volume 1)}.\hskip 1em plus 0.5em minus
  0.4em\relax Cham: Springer International Publishing, 2016, ch. RotorS---A
  Modular Gazebo MAV Simulator Framework, pp. 595--625. [Online]. Available:
  \url{http://dx.doi.org/10.1007/978-3-319-26054-9_23}
\BIBentrySTDinterwordspacing

\end{thebibliography}

\end{document}